\documentclass{article}
\usepackage[T1]{fontenc}
\usepackage{microtype}
\usepackage{graphicx}
\usepackage{subfigure}
\usepackage{booktabs} 
\usepackage{enumitem}
\usepackage{pifont}
\usepackage{amsthm}
\usepackage{amssymb}
\setlist{nosep} 

\usepackage{hyperref}
\usepackage{mathrsfs,amsmath}

\usepackage[accepted]{icml2021}

\usepackage{amsmath,amsfonts,bm}

\newcommand{\T}{^\mathrm{T}}
\newcommand{\tr}{\mathrm{tr}}

\def\eqref#1{equation~\ref{#1}}

\def\1{\bm{1}}

\def\rd{{\textnormal{d}}}

\def\vb{{\bm{b}}}

\def\vg{{\bm{g}}}

\def\vu{{\bm{u}}}

\def\vw{{\bm{w}}}
\def\vx{{\bm{x}}}

\def\vz{{\bm{z}}}

\def\mF{{\bm{F}}}

\def\mH{{\bm{H}}}
\def\mI{{\bm{I}}}

\def\mU{{\bm{U}}}

\def\mW{{\bm{W}}}
\def\mX{{\bm{X}}}
\def\mY{{\bm{Y}}}
\def\mZ{{\bm{Z}}}

\DeclareMathAlphabet{\mathsfit}{\encodingdefault}{\sfdefault}{m}{sl}
\SetMathAlphabet{\mathsfit}{bold}{\encodingdefault}{\sfdefault}{bx}{n}

\def\sI{{\mathbb{I}}}

\def\sR{{\mathbb{R}}}

\def\sX{{\mathbb{X}}}

\newcommand{\E}{\mathbb{E}}

\newcommand\numberthis{\addtocounter{equation}{1}\tag{\theequation}}

\newcommand{\xmark}{\ding{55}}%

\newtheorem{theorem}{Theorem}
\newtheorem{lemma}{Lemma}
\newtheorem{proposition}{Proposition}

\newtheorem{remark}{Remark}
\newtheorem{assumption}{Assumption}
\usepackage{graphicx}
\usepackage{hyperref}
\usepackage{wrapfig}
\usepackage{url}
\usepackage{algorithm}
\usepackage{algorithmic}
\usepackage{graphicx}
\usepackage{todonotes}
\usepackage{dsfont}
\usepackage{amsmath,amsfonts,amssymb}
\usepackage{cleveref}

\icmltitlerunning{Large-Scale Multi-Agent Deep FBSDEs}

\begin{document}

\twocolumn[
\icmltitle{Large-Scale Multi-Agent Deep FBSDEs}



\begin{icmlauthorlist}
\icmlauthor{Tianrong Chen}{goo}
\icmlauthor{Ziyi Wang}{ml}
\icmlauthor{Ioannis Exarchos}{to}
\icmlauthor{Evangelos Theodorou}{ml,ae}
\end{icmlauthorlist}

\icmlaffiliation{to}{Department of Computer Science, Stanford University, Stanford, USA}
\icmlaffiliation{goo}{School of Electrical and Computer Engineering, Georgia Institute of Technology, Atlanta, USA}
\icmlaffiliation{ml}{Center for Machine Learning, Georgia Institute of Technology, Atlanta, USA}
\icmlaffiliation{ae}{School of Aerospace Engineering, Georgia Institute of Technology, Atlanta, USA}

\icmlcorrespondingauthor{Tianrong Chen}{tchen429@gatech.edu}
\icmlkeywords{Machine Learning, ICML}
\vskip 0.3in
]



\printAffiliationsAndNotice{}  

\begin{abstract}
In this paper we present a scalable deep learning framework for finding Markovian Nash Equilibria in multi-agent stochastic games using fictitious play. The motivation is inspired by theoretical analysis of Forward Backward Stochastic Differential Equations (FBSDE) and their implementation in a deep learning setting, which is the source of our algorithm's sample efficiency improvement. By taking advantage of the permutation-invariant property of agents in symmetric games, the scalability and performance is further enhanced significantly. We showcase superior performance of our framework over the state-of-the-art deep fictitious play algorithm on an inter-bank lending/borrowing problem in terms of multiple metrics. More importantly, our approach scales up to 3000 agents in simulation, a scale which, to the best of our knowledge, represents a new state-of-the-art. We also demonstrate the applicability of our framework in robotics on a belief space autonomous racing problem.
\end{abstract}
\section{Introduction}
Stochastic Differential Games (SDG) represent a framework for investigating scenarios where multiple players make decisions in a stochastic environment. The theory of differential games dates back to the seminal work of \citet{Isaacs1965} studying two-player zero-sum dynamic games, with the stochastic extension first appearing in \citet{Kushner1969}. A key step in the study of games is to obtain the Nash equilibrium among players \citep{osborne1994course}. A Nash equilibrium represents the solution of non-cooperative games where two or more players are involved. At the equilibrium, each player cannot gain any benefit by modifying his/her own strategy given opponents' strategy. In the context of adversarial multi-objective games, the Nash equilibrium can be represented as a system of coupled Hamilton-Jacobi-Bellman (HJB) equations when the system satisfies the Markovian property. Analytical solutions exist only for few special cases. Therefore, obtaining the Nash equilibrium solution is usually done numerically, and this can become challenging as the number of states/agents increases. Despite extensive theoretical work \cite{buckdahn2008stochastic,ramachandran2012stochastic}, the algorithmic part has received less attention and is limited to special cases of differential games (e.g., \citet{Duncan2015}), or suffers from the curse of dimensionality \citep{Kushner2002}. Nevertheless, SDG has a variety of applications including in robotics, autonomy, economics and management. Relevant examples include \citet{mataramvura2008risk}, which formulates portfolio management as a SDG in order to obtain a market portfolio that minimizes the convex risk measure of a terminal wealth index value, as well as \citet{prasad2004competitive}, which investigates optimal advertising spending in duopolistic settings via SDG. 

Researchers have also been solving differential games via Reinforcement Learning (RL) \cite{harmon1995reinforcement}. \citet{kamalapurkar2016model} seeks to combine differential game theory and RL to solve cooperative control problems in which agents are coupled by graphs. Multi-agent RL (MARL) is an extension of RL, which is a more complex task due to internal interactions between agents and external interactions with environment. Independent learning \citep{tan1993multi} is an approach which assumes other agents are part of environment, and it suffers from the unstable learning problem. Training agents by augmenting their states and actions is called centralized training and executing \citep{yang2019alpha}, and scalability of it is limited.
Another method is centralized training and decentralized execute (CTDE), however the challenge therein lies on how to decompose value function in the execute phase for value-based MARL. \citet{sunehag2018value, zhou2019factorized, rashid2018qmix, son2019qtran,mahajan2019maven} achieve remarkable success in this vein. However, RL usually adopts model-free learning scheme in discrete time. Here we leverage the HJB PDE which contains information of the dynamics to develop data-efficient algorithm in continuous time. The decomposition of value function is characterized by decoupled property of HJB PDE by modifying nonlinear Feynman-Kac Lemma \citep{karatzas1991brownian} in multi-agent scenario.

Another vein of related research marries the mathematical formulation of a differential game with nonlinear PDE. This motivates the algorithmic development for differential games that combines elements of PDE theory with deep learning. Recent encouraging results  \citep{han2018solving,raissi2018forward} in solving nonlinear PDEs with deep learning illustrate the scalability and numerical efficiency of neural networks. The transition from a PDE formulation to a trainable neural network is done via a system of Forward-Backward Stochastic Differential Equations (FBSDEs). Specifically, certain PDE solutions are linked to solutions of FBSDEs, and the latter can be solved using a suitably defined neural network architecture. This is known as the \textit{Deep FBSDE} approach.   \citet{han2018solving,pereira2019learning,wang2019deep} utilize various deep neural network architectures to solve such stochastic systems. 
However, these algorithms address single agent dynamical systems. Two-player zero-sum games using FBSDEs were initially developed in \citet{exarchos2019stochasticgames} and transferred to a deep learning setting in \citet{wang2019deepgame}. Recently, \citet{hu2019deep} brought deep learning into fictitious play to solve multi-agent non-zero-sum games. \citet{han2019deep} introduced the Deep FBSDEs to a multi-agent scenario and the concept of fictitious play. Furthermore, \citet{han2020convergence} gives the convergence analysis of the framework.
Our work mainly focuses on stochastic differential game with homogeneous agents which can also be seen as the case of mean field game with finite number of agents. Under mild assumptions, the approximation error of mean field game is of order $N^{-1/(d+4)}$ \cite{carmona2013probabilistic} where $N$ is the number of agents and $d$ is the state dimension of individual agent. Our algorithm can mitigate the approximation error when the number of agents is moderate.

In this work, we first extend the theoretical analysis in \citet{han2020convergence} to the scenario where the forward process includes a drift term. Suggested by the theoretic result, we integrate importance sampling \citep{bender2010importance} into the existing framework. Furthermore, by leveraging the property of symmetric game, we introduce an invariant feature representation to reduce the complexity and significantly increase the number of agents the framework can handle. The main contribution of our work is threefold:
\begin{enumerate}
  \item We theoretically analyze FBSDE with drift term under fictitious play problem setup, and explain the efficiency of importance sampling analytically and numerically.
  \item We introduce an efficient Deep FBSDE framework for solving stochastic multi-agent games via fictitious play that outperforms the current state-of-the-art in performance and runtime/memory efficiency on an inter-bank lending/borrowing example. We demonstrate that our approach scales to a much larger number of agents (up to 3,000 agents, compared to 100 in existing works). To the best of our knowledge, this represents a new state-of-the-art.
  \item We showcase the applicability of our framework to robotics on a belief-space autonomous racing problem which has larger individual control and state space. The experiments demonstrate that the BSDE has decoupled property and provides the possibility of applications for competitive scenario, and the model can extract informative features with partially observed input data.
\end{enumerate}
The rest of the paper is organized as follows: in $\S$\ref{sec:prelim} we present the notation and mathematical preliminaries. In $\S$\ref{sec:SDFP-FBSDE} we introduce the Scaled Deep Fictitious Play FBSDE (SDFP-FBSDE) algorithm, with simulation results following in $\S$\ref{sec:simulation}. We conclude the paper in $\S$\ref{sec:conclusion}.

\section{Notation and Preliminaries}
\label{sec:prelim}

\begin{table}[h]
	\caption{Mathematical notation.}
	\label{table:prelim}
	\vskip 0.1in
	\begin{center}
	\begin{small}
	\begin{tabular}{r|cr}
	\toprule
	\textsc{Characters} & \textsc{Definitions}  \\
	\midrule
	$\mX$    & \text{quantities for all agents}\\ [1pt]
	$X^i/X_i$ & quantities for $i$th agent\\ [1pt]
	$\mX_{-i}$ & quantities from all agents except $i$th\\[1pt]
	$\vx$ & realization of $\mX$\\[1pt]
	\bottomrule
	\end{tabular}
	\end{small}
	\end{center}
	\vskip -0.1in
\end{table}

Throughout the paper the notation will be:
\begin{enumerate}
   \item Bold letters denote quantities for all players\footnote[1]{Agent and player are used interchangeably in this paper}
   \item (sup)subscript $i$ indicates quantities for player $i$.
   \item Bold letters with (sup)subscript $-i$ indicates quantities for all players except $i$.
   \item Lower case letter is the realization of the upper case random variable.
\end{enumerate}
Table.\ref{table:prelim} contains examples for the notation defined above. The baseline used in this paper refers to \citet{han2019deep}.

\subsection{HJB PDE and FBSDE}
Fictitious play is a learning rule \citep{brown1951iterative} where each player presumes other players' strategies to be fixed. An $N$-player game can then be decoupled into $N$ individual decision-making problems which can be solved iteratively over $M$ stages. When each agent converges to a stationary strategy at stage $m-1$, this strategy will become the stationary strategy for other players at stage $m$. We consider a $N$-player non-cooperative SDG with dynamics: 
\begin{align*}
\vspace{-5pt}
		\rd  \mX_t&=\big(f(\mX_t, t) +G(\mX_t, t)\mU(\mX_t)\big)\rd t+\Sigma(
		\mX_t, t)\rd\boldsymbol{W}_t\\ \mX_{t_0}&=\vx_{t_0},
		\numberthis \label{state_process}
\end{align*}
where $\mX=(X_1, X_2,\dots,X_N)$ is a vector containing the state process of all agents generated by their controls $\mU=(U_1,U_2,\dots,U_N)$  defined on the space $\mathcal{X}$ and $\mathcal{U}, $ with $X_i \in \mathcal{X}_i \subseteq \mathbb{R}^{n_x}$ and $U_i\in\mathcal{U}_i\subseteq\mathbb{R}^{n_u}$. Here, $f:\mathcal{X} \times [t_0,T] \rightarrow \mathcal{X}$ represents the drift dynamics, $G:\mathcal{X} \times [t_0,T] \rightarrow \mathcal{X}\times\mathcal{U}$ represents the actuator dynamics, and $\Sigma: \mathcal{X}\times[t_0,T] \rightarrow \mathcal{X}\times\mathbb{R}^{n_w}$ represents the diffusion term. 
The state process is also driven by an $n_w$-dimensional independent Brownian motion $W_i$, and denoted $\mW=(W_1, W_2, \dots, W_{N})$.\\
Given the other agents' strategies, the stochastic optimal control problem for agent $i$ under the fictitious play assumption is defined as minimizing the expectation of the cumulative cost functional $J^i_t$
	\begin{equation}\begin{split}\label{cumulative_loss}
		&J^i_t(\mX,U_{i,m};\mU_{-i,m-1})=\\
		&\E\left[g(\mX_T)+\int_t^T C^i(\mX_\tau,U_{i,m}(\mX_\tau);\mU_{-i,m-1})\rd \tau\right],
		\end{split}
	\end{equation}
   where $g:\mathcal{X} \rightarrow \sR^{+}$ is the terminal cost, and $C^i:\mathcal{X} \times \mathcal{U} \rightarrow \sR^{+}$ is the running cost for the $i$-th player. In this paper we assume that the running cost is of the form $C^i(\mX,U_{i})=q^i(\mX)+\frac{1}{2}U_{i}^\mathrm{T}RU_{i} +  \mX^\mathrm{T}Q U_{i}$. We use the double subscript $U_{i,m}$ to denote the control of agent $i$ at stage $m$. We can define value function of each player as
   \begin{equation}
   \begin{split}
           &V^i_t(\mX_t)=\inf \limits_{u_{i,m} \in \mathcal{U}_i}\left[J_t^i(\mX,U_{i,m};\mU_{-i,m-1})\right]\\
           &V^i_T(\mX_T)=g(\mX_T).\label{Value_function}
        \end{split}
      \end{equation}
	Assume that $V^i_t$ in eq.~(\ref{Value_function}) is once differentiable w.r.t. $t$ and twice differentiable w.r.t. $\mX$. Here we drop the functional dependencies for simplicity. Then, standard stochastic optimal control theory leads to the HJB PDE:
\begin{align}
    &V^i_t+\inf_{u^i \in \mathcal{U}_i}\left\{V_x^{i\mathrm{T}}(f+G \mU)+C^i\right\}+\frac{1}{2}\tr(V^i_{xx}\Sigma\Sigma\T)=0 \label{vanilla_form_HJB}
\end{align}
The function inside the infimum is known as the Hamiltonian function $\mH(t,\mX,\mU,Z^i)$. $Z^i$ is known as adjoint states and is defined as $Z^i=\Sigma\T V_x^i$ in literature \citep{yong1999stochastic}. If the optimal control can be obtained, by plugging back to eq.\ref{vanilla_form_HJB}, one can have,
	\begin{equation}\label{Final_form_HJB}
	\begin{aligned}
    &V^i_t+h+V^{i\mathrm{T}}_x (f+G \mU_{0,*})+\frac{1}{2}\tr(V^i_{xx}\Sigma\Sigma\T)=0,
	\end{aligned}    
	\end{equation}
where $h^i=C^{i*}+G\mU_{*,0}$. The first and second subscript of $\mU$ represents for the control taken by $i$th and $-i$th agent respectively. The $*$ denotes the optimal control computed from $U_{i}^*=-R^{-1}(G_i^\mathrm{T}V^i_x+Q_i^\mathrm{T}\mX)$ and the 'zero' represents for taking zero control $\mU_{-i}=0$. For instance, $\mU_{0,*}$ denotes the augmentation of $U_{i}=0$ and $\mU_{-i}=\mU_{-i}^*$ as $(U_1^*, \cdots, U_{i-1}^*, 0, U_{i+1}^*,\cdots, U_N^*)$ and $\mU_{*, 0} = (0,\cdots, 0,U_i^*, 0,\cdots, 0)$. The derivation is included in Appendix \ref{Appendix:hjb_derivation} for completeness. The value function $V^{i}_t$ in the HJB PDE can be related to a set of FBSDEs by non-linear Fayman-Kac Lemma \citet{karatzas1991brownian},
	\begin{equation}\label{fnc:FBSDE}
	\begin{aligned}
        &\rd  \mX_t=(f +G\mU_{0,*})\rd t+\Sigma\rd\boldsymbol{W}_t,\, \mX_{t_0}=\vx_{t_0} &\text{ (FSDE)},\\ 
        &\rd Y^i_t=-h^i_t\rd t+Z^i_t \rd W_t,\, Y^i_T=g(\mX_T), &\text{ (BSDE)}, \\
	\end{aligned}    
	\end{equation}
where the backward process corresponds to the value function, and $t\in[t_0,T]$. The derivation can be found in Appendix \ref{Appendix:FBSDE_derivation}. Here we denote value function $V$ as $Y$ in order to be aligned with classic FBSDE literature. 
\begin{remark}
     The problem of solving HJB PDE (\ref{Final_form_HJB}) will be transformed to solve a FBSDE system. (\ref{fnc:FBSDE}). The $i$th player will provide zero control in the drift term in the forward process, and the rest of agents will execute the optimal policy according to the Value function.
    
    The HJB PDE (\ref{Final_form_HJB}) can be solved by Numerical PDE solvers based on finite element and finite difference methods \cite{greif2017numerical}. Howerver, these methods do not scale beyond very few dimensions, as they rely on discretization.
\end{remark}
\subsection{Deep Fictitious Play FBSDE Controller} \label{sec:DFP Controller}
In \citet{han2019deep}, the FBSDE system in (\ref{fnc:FBSDE}) is solved by a neural network where each agent has a Deep Fictitious Play FBSDE Controller (DFP-FBSDE) to generate the policy. 

The neural network architecture is described as Fig.\ref{fig:architecture}, where the backbone is a unique fully connected neural network followed by a BatchNorm module \citep{ioffe2015batch} (hereafter we shorthand this backbone as FC) at each timestep. Since \citet{han2019deep} do not apply feature extraction, it is equivalent to setting feature extractor as identity mapping: $\hat{\mF}_t=\mI\cdot \mX_t$. The output of the network is $Z$ component in functions (\ref{fnc:FBSDE}). The use of a neural network allows for the BSDE to be propagated forward (red path) from the initial condition predicted by another FC along with the FSDE (green 
path).
\begin{figure}[t]
  \centering
  \includegraphics[width=1.0\linewidth]{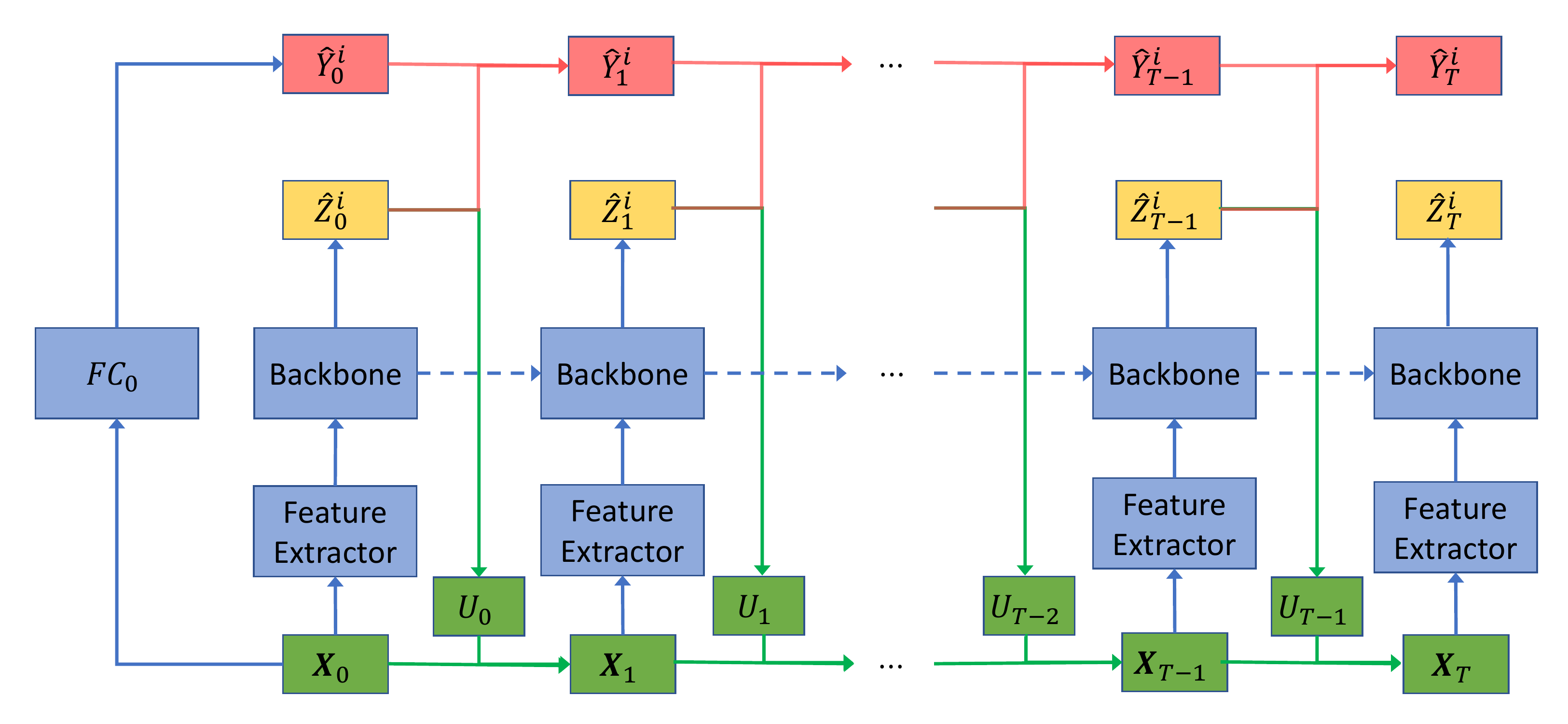}
  \vskip -0.2in
  \caption{
    FBSDE Network for a single agent. The dashed arrow indicates hidden states propagation if LSTM is chosen as backbone. The dash arrow would disappear when FC is chosen.
  }%
  \label{fig:architecture}%
  \vskip -0.2in
\end{figure}

The training set of DFP-FBSDE contains randomly selected initial states $\vx_0$, and the training labels are defined as the value function at terminal time $Y^i_T$(\ref{fnc:FBSDE}) which depends on random variable $\mX_T$. The training loss (\ref{fnc:eval-loss}) is defined as the mean square error between the predicted value $\hat{Y}_T^i$ propagated by BSDE and the true terminal value. The metrics used on training/test data are summarized in Appendix \ref{appendix:metrics}.

By collecting the FBSDEs from all agents, the final DFP-FBSDE framework will be formed and the architecture can be found in Appendix \ref{appendix:NN_arch}. Furthermore, \citet{pereira2019learning} extends Deep FBSDE to robotics applications and \citet{han2019deep} applys Deep FBSDE for searching Nash-Equlibrium with convergence proof \citep{han2020convergence}.

\section{SDFP-FBSDE}\label{sec:SDFP-FBSDE}
In this section we propose a novel SDFP-FBSDE algorithm. Motivated by the convergence analysis of DFP-FBSDE, we introduce an Importance Sampling term to facilitate exploration and accelerate convergence. Scalability is further increased by Invariant Layer Embeddings, leading to an order of magnitude lower time and memory complexity.

\subsection{Motivation of Importance Sampling (IS)}\label{sec:convergence}

We first analyze FBSDE (\ref{fnc:FBSDE}) with drift term by extending previous work \cite{han2020convergence} where the drift term is ignored. The analysis focuses on one representative agent. Therefore, the agent index is omitted in this subsection.
\begin{assumption}\label{asmpt:dyn}
	There exists a measurable function $\phi:[0,T]\times \mathcal{X}\rightarrow \mathcal{X}$ and $\Gamma:[0,T]\times \mathcal{U}\rightarrow\mathcal{X}$ so that $\Sigma(t,\mX)\phi(t,\mX)=f(t,\mX)$, and $\Sigma(t,\mX)\Gamma(t,\mX)=G(t,\mX)$.
\end{assumption}
\begin{assumption}\label{asmpt:liptz}
	For a general FBSDE system,
	\begin{equation}\label{fnc:general-FBSDE}
	\begin{aligned}
        &\mX_T^{t,\vx}=\vx+\int_{t}^{T}\mu_s\rd s+\int_{t}^{T}\Sigma_s\rd\mW_s &\text{ (FSDE)},\\ 
        & Y_{t}^{T,\vx}=g(\mX_T^{t,\vx})-\int_{t}^{T}H_s\rd s+\int_{t}^{T}Z_s \rd \mW_s &\text{ (BSDE)}, \\
	\end{aligned}    
	\end{equation}
The drift function $\mu_s$, $H_s$ and diffusion function $\Sigma$ in (\ref{fnc:general-FBSDE}), terminal objective function $g(\cdot)$, and control function $U(\cdot)$ satisfy Lipschitz continuous properties with Lipschitz constants $\mu_x,\mu_u,\Sigma_x,H_x,H_z,g_x,u_x$ respectively. The detailed and formal description can be found in Appendix \ref{appendix:liptz}.
\end{assumption}
\begin{lemma}
	\label{Lem: Y Bound}
    Denote $(\mX_s^{t,\vx},Y_s^{t,\vx},Z_s^{t,\vx})_{t\leq s\leq T}$ as the solution for the FBSDE system (\ref{fnc:general-FBSDE}) satisfying assumptions \ref{asmpt:dyn} and \ref{asmpt:liptz}. Denote the difference of $Y$ component at two different states $\vx_1$ and $\vx_2$ as:
	\begin{equation}
		\begin{aligned}
		    \delta \mX_t&=\mX_t^{t_0,\vx_1}-\mX_t^{t_0,\vx_2}, \delta Y_t=Y_t^{t_0,\vx_1}-Y_t^{t_0,\vx_2}.\\
		    \delta Z_t&=Z_t^{t_0,\vx_1}-Z_t^{t_0,\vx_2}
		\end{aligned}
	\end{equation}	
	Then we can have:
	\begin{equation}\label{eq:diff Y}
		\begin{aligned}
			|\delta Y_{T}|^2 &\leq L_1|\vx_1-\vx_2|^2, |\delta Y_{t_0}|^2 \leq L_2|\vx_1-\vx_2|^2,\\
		\end{aligned}
	\end{equation}	
	Where $L_1$ and $L_2$ are defined as:
	\begin{equation}
		\begin{aligned}
		    L_1&=g_xe^{\xi T}\\
			L_2&= e^{H_z(T-t_0)}\left[g_x e^{\xi(T-t_0)}+H_x\frac{e^{\xi(T-t_0)}-1}{\Sigma^{-1}_xH_z^{-1}}\right],\\
			&\xi=I+\mu_x+\mu_u u_x+\Sigma_x,
		\end{aligned}
	\end{equation}
	Following arguments in \cite{ma2002representation}, one further has,
	\begin{equation}
		\begin{aligned}
            ||Z_t||_S^2\leq ||\Sigma||_S^2 ||\nabla_x Y_t||_S^2\leq M_\Sigma L_2
		\end{aligned}
	\end{equation}
	Where $\mu_x,\mu_u,u_x,\Sigma_x,H_x,H_z, g_x$ are Lipschitz constants defined in Assumption.\ref{asmpt:liptz}. $M_\Sigma$ is the upper bound of $\Sigma$. The proof can be found in Appendix \ref{lem:lemma1-proof}.
\end{lemma}
Lemma \ref{Lem: Y Bound} bridges the connection between the states and their corresponding value functions. Note that the training labels $Y_T$ are not fixed since it is a  function of $\mX_T$ driven by stochastic diffusion $\mW_t$ which is distinct at each training iteration. Furthermore, the  Lemma \ref{Lem: Y Bound} suggests that scope of the training dataset $[\vx_0,Y_T^i]$ and initial values functions $Y_{t_0}$ can be controlled by Lipschtiz constants $\mu_x, \mu_u,H_x, H_z,u_x$, since Lipschtiz constants $\Sigma_x, g_x$ are determined as long as the stochastic dynamics and terminal loss function are defined.

Increasing data complexity is crucial element for improving model performance in RL and supervised learning. In RL, agents are encouraged to explore more states by adding the entropy term in cost function \cite{haarnoja2018soft} or adding artificial stochasticity in observed states \cite{fortunato2017noisy} to gain complex training data for the learning of action-value function. In supervised learning, connections between deep neural networks' generalization and complexity of dataset also has attracted considerable attention as a principled tool to explain deep learning(DL) \citep{schmidt2018adversarially},\citep{rozen2019diversify}. In practice, Data-Augmentation is an efficient approach adopted to improve the generalization performance of a deep learning model \cite{cubuk2018autoaugment},\cite{fawzi2016adaptive}. Motivated by the success of previous work in DL and RL community, we would like to modify the FBSDE system (\ref{fnc:FBSDE}) to improve the complexity of training label $\delta Y_T$ and $\delta Y_{t_0}$ (\ref{eq:diff Y}) by encouraging the exploration of agents while solving the  same HJB PDE (\ref{Final_form_HJB}).

\begin{theorem}\label{thm:impt-samp}
    \cite{bender2010importance}: 
    Let $(X_s^{i,t,x},Y_s^{i,t,x},Z_s^{i,s,x})$ be the solution of the FBSDE system (\ref{fnc:FBSDE}) for $i$th agent, and let $K_s: [0,T]\times \Omega \rightarrow \mathbb{R}^{n_x}$ be any bounded and square integrable process for $i$th agent. Consider the forward process whose drift term is modified by $K_s$
	\begin{equation}\label{fnc:impt-FSDE}
	\begin{aligned}
        &\rd  \tilde{\mX}_s=[\mu_s+\Sigma K_s]\rd s+\Sigma\rd\boldsymbol{W}_s,\, \tilde{\mX}_t=\vx_t,\\
	\end{aligned}    
	\end{equation}
	along with the corresponding BSDE
	\begin{equation}\label{fnc:impt-BSDE}
        \rd \tilde{Y}^i_s=[-h^i_s+\tilde{Z}_sK_s]\rd s+\tilde{Z}^i_s \rd W_s,\, Y^i_T=g(\mX_T).
	\end{equation}
	Here we  denote $(\tilde{X}_s^{i,t,x},\tilde{Y}_s^{i,t,x},\tilde{Z}_s^{i,s,x})$ as the solution for modified FBSDE system (\ref{fnc:impt-FSDE},\ref{fnc:impt-BSDE}). For all $s\in [t,T]$, $(\tilde{X}_s^{i,t,x},\tilde{Y}_s^{i,t,x},\tilde{Z}_s^{i,s,x})=(X_s^{i,t,x},Y_s^{i,t,x},Z_s^{i,s,x})$ a.s. If $(\tilde{Y}_s^{i,t,x},\tilde{Z}_s^{i,s,x})$ are defined as $(\tilde{V}^i, \Sigma^{\mathrm{T}}\tilde{V}_x^i)$ with $\tilde{V}^i$ being the solution to \ref{Final_form_HJB}, then $V^i \equiv \tilde{V}^i$ a.e.
\end{theorem}
Theorem \ref{thm:impt-samp} is also known as IS in the literature \cite{exarchos2018stochastic}. By modifying the FSDE and BSDE simultaneously, the HJB PDE will be solved identically almost everywhere. We further bound $|\delta Y_T|^2$, $|\delta Y_{t_0}|^2$ and $||Z_t||_S^2$ with constant $\tilde{L}_1$ and $\tilde{L}_2$ when IS is applied in the FBSDE system in Lemma \ref{Lem:lemma1-impt-samp}. The detailed description and proof can be found in Appendix \ref{Lem: proof-lemma1-impt-samp}.
\begin{theorem}
	\label{prop: diff Y bound}
    Denote $(\tilde{\mX}_s^{t,\vx},\tilde{Y}_s^{t,\vx},\tilde{Z}_s^{t,\vx})_{t\leq s\leq T}$ is the solution for the FBSDE system with IS (\ref{fnc:impt-FSDE},\ref{fnc:impt-BSDE}), and $(\mX_s^{t,\vx},Y_s^{t,\vx},Z_s^{t,\vx})_{t\leq s\leq T}$ is the solution for the FBSDE system (\ref{fnc:general-FBSDE}). and they satisfy the assumption (\ref{asmpt:dyn},\ref{asmpt:liptz}).
	Then given the identical training data $\mX_0$ for FBSDE w/ and w/o IS, combining Lemma \ref{Lem: Y Bound} and Lemma \ref{Lem:lemma1-impt-samp}, one can have,
	\begin{equation}
		\begin{aligned}
			\max |\delta Y_{T}|^2 &\leq \max|\delta \tilde{Y}_{T}|^2,\\
			\max |\delta Y_{0}|^2 &\leq \max|\delta \tilde{Y}_{0}|^2
		\end{aligned}
	\end{equation}
\end{theorem}
Theorem \ref{prop: diff Y bound} hints that, given identical training data , the model equipped with IS could obtain more diverse and complex training dataset, which means the generalization performance can potentially be enhanced. Fig.\ref{fig:limited-dataset} testifies the performance difference for the model with and without IS over identical dataset whose size is limited on the inter-bank problem ($\S$\ref{subsec:interbank}). The metric is chosen as the evaluation loss (\ref{fnc:eval-loss}) and Relative Square Error (RSE) loss (\ref{fnc:RSE-loss}). We can find IS improve the performance for both backbones. In order to highlight the influence of IS term $\Sigma K_s$ and $\tilde{Z}K_s$, we detached/stopped the gradient of them. In other words, the IS term will be treated as a constant during training phase, and the rapid convergence is not the result of the additional gradient path introduced by IS.

As a value-based optimal decision algorithm, a natural question will be how the policy performs based on the learnt value function. Proposition \ref{prop:convergence rate} demonstrates the convergence rate of Deep FBSDE policies in fictitious play.
\begin{proposition}\label{prop:convergence rate}
    \cite{han2020convergence} Denote $\delta \vu_t^m=\vu_t^m-\vu_t^{*}$ as difference between the policy obtained at $m$th stage of fictitious play at time step $t$ and the optimal policy, then,
	\begin{equation}
		\begin{aligned}
			\int_{0}^{T}\mathbb{E}|\delta \vu_t^{m+1}|^2\leq \eta(\lambda) \int_{0}^{T}\mathbb{E}|\delta \vu_t^{m}|^2
		\end{aligned}
	\end{equation}
	Where $\eta(\lambda)$ is the convergence rate.
\end{proposition}
The theoretical analysis of proposition \ref{prop:convergence rate} uncovers the convergence property of fictitious play of FBSDE system. However, choice of constant $\lambda$ is tricky in order to obtain strict converge: $0\leq \eta(\lambda)< 1$. We demonstrate the numerical value of $\frac{1}{\eta(\lambda)}$ for Deep FBSDE model w/ and w/o IS with FC backbone \cite{han2018solving} and LSTM backbone \cite{pereira2019learning} in Fig.\ref{fig:cvg-rate}. For the same backbone, IS accelerates the convergence of the policy significantly.
\begin{figure}[t]%
  \centering
  \includegraphics[width=1.1\linewidth]{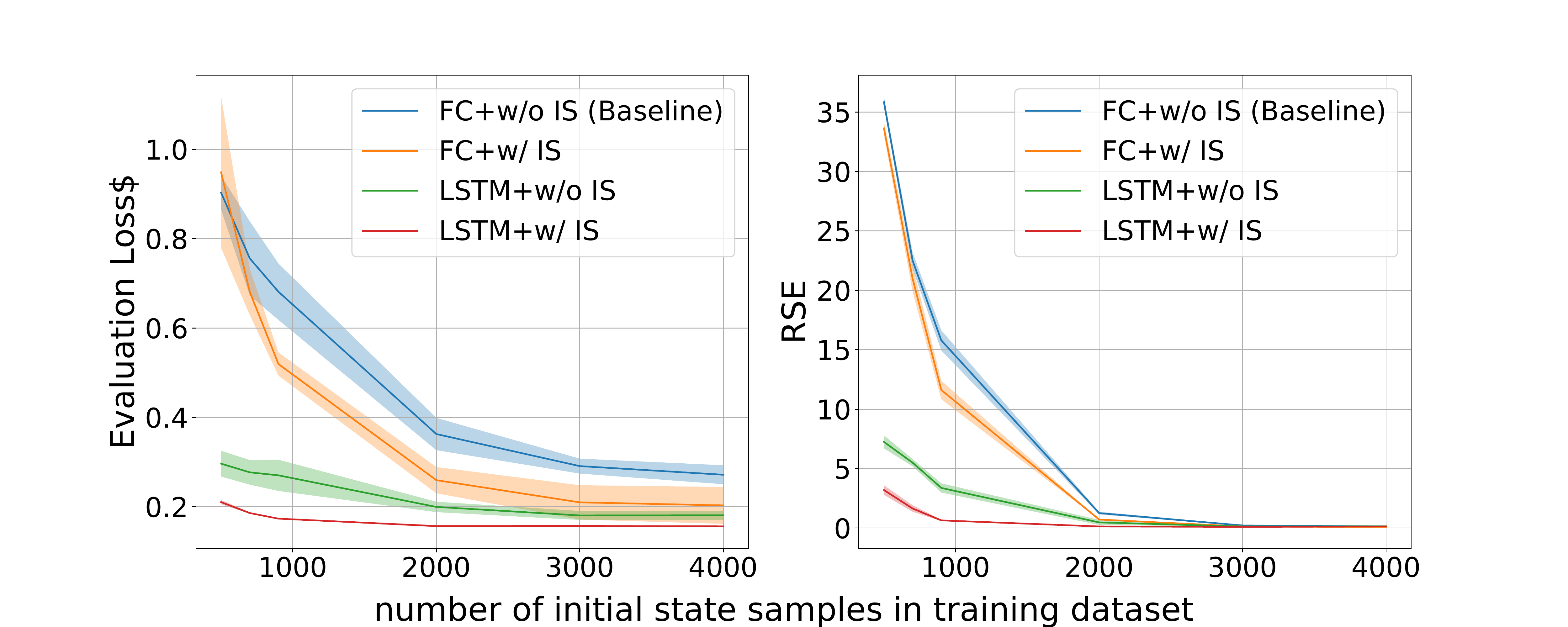}
  \vskip -0.07in
  \caption{
    Performance difference of DFP-FBSDE w/ and w/o importance samping over limited training dataset. The simulation is executed on 100 agents inter-bank game.
  }%
  \label{fig:limited-dataset}
  \vskip -0.15in
\end{figure}
\begin{figure}[t]%
  \centering
  \includegraphics[width=1.0\linewidth]{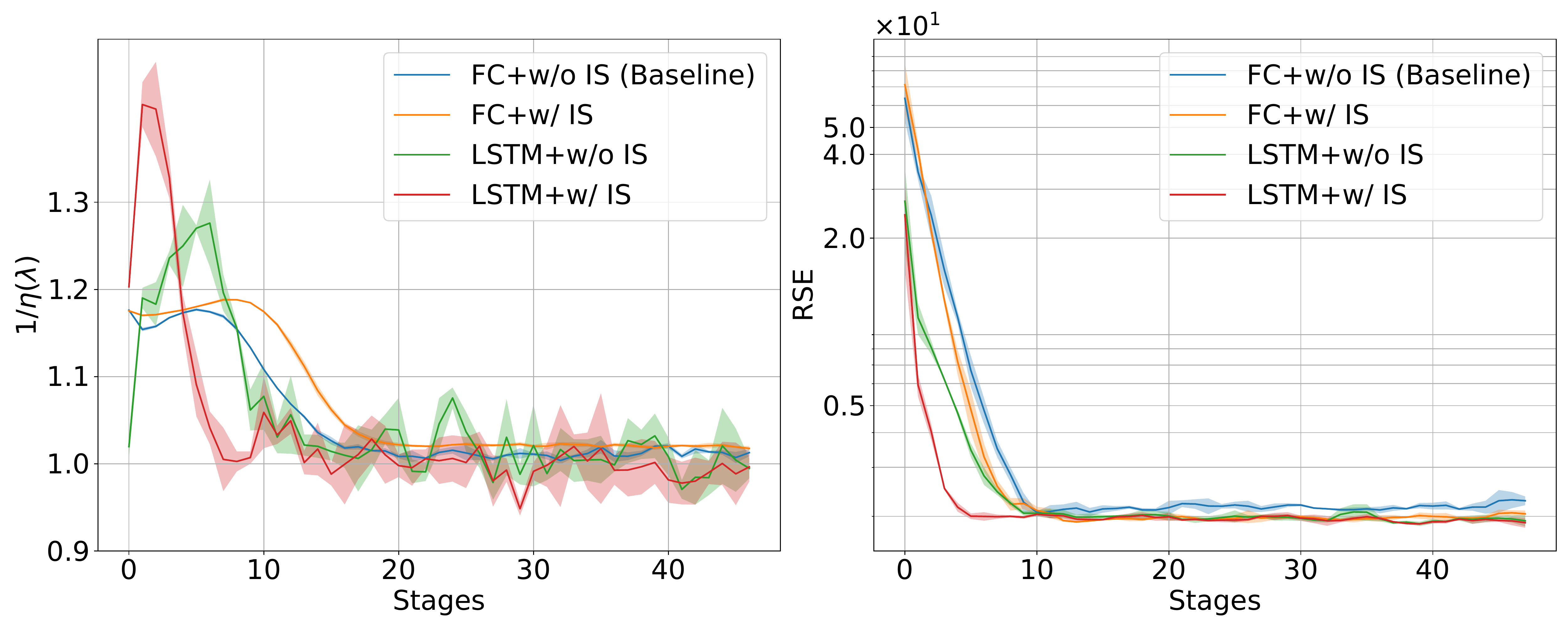}
  \vskip -0.2in
  \caption{
    Numerical value of convergence rate $\frac{1}{\eta(\lambda)}$ and the model performance evalued by RSE. The simulation is executed on 100 agents inter-bank game.
  }%
  \label{fig:cvg-rate}
  \vskip -0.15in
\end{figure}

\subsection{Mitigating Curse of Many Agents}\label{sec:invar-layer}

Scalability is a crucial criterion of RL. In DFP-FBSDE, each agent is equipped with policy computed by a distinct neural network. As the number of agents increases, the number of neural network copies would increase correspondingly. Meanwhile, the size of each neural network should be enlarged to gain enough capacity to capture the representation of many agents, leading to the infamous curse of many agents \cite{wang2020breaking}. This limits the scalability of prior works. However, one can mitigate the curse of many agents in this case by taking advantage of the symmetric game setup. We summarize merits of symmetric game as:
\begin{enumerate}
    \item Since all agents have the same dynamics and cost function, only one copy of the network is needed. The strategies of other agents can be inferred by applying the same network.
    \item Thanks to the symmetric property, we can apply the IL embedding \citep{zaheer2017deep} to extract invariant features.
\end{enumerate}
\textbf{Sharing one network:} Here we introduce some techniques to improve the time and memory complexity. When sharing same network, we can query the policy of other agents simply by inferring own freeze network at $m-1$ stage. It's important to note that querying the policy of other agents should not introduce additional gradient paths, which significantly reduces the memory complexity. 

When querying other agents' strategy, one can either iterate through each agent or feed all agents' states to the network in a batch. Here we denote them as iterative scheme and batch scheme respectively. The latter approach reduces the time complexity by adopting the parallel nature of modern GPU but requires $\mathcal{O}(N^2)$ memory complexity rather than $\mathcal{O}(N)$ for the first approach.

\textbf{Invariant Layers (IL):}
A comprehensive introduction of invariant network can be found in \cite{zaheer2017deep}. Later applications in robotics \cite{shi2020neural} and RL \cite{liu2020pic}, \cite{wang2020breaking} verify the success of this technique. The memory complexity can be further reduced with an IL embedding \citep{zaheer2017deep} in our work. The IL utilizes an averaging function along with the features in the same set (known as feature averaging) to render the network invariant to permutation of agents. We apply the IL on $\mX_{-i}$ which is insensitive to permutation and concatenate the resulting features to the features extracted from $X_i$. However, vanilla IL embedding does not reduce the memory complexity, because the dimension of feature space is commonly high even though the dimension of averaged feature is small. Thanks to the symmetric problem setup, one can apply a technique to reduce the IL memory complexity form $\mathcal{O}(N^2)$ to $\mathcal{O}(N)$. A detailed introduction to the IL and our implementation techniques can be found in Appendix \ref{Appendix:invar-layer}.

Furthermore, the invariant representation, especially for feature averaging, can increase the performance of deep learning model theoretically \cite{lyle2020benefits} when features have the invariant property. Learning the invariant representation for $\mX_{-i}$ will become harder when the number of agents increases. 
\begin{figure}[t]
  \centering
    \vskip 0.1in
  \includegraphics[width=1.0\linewidth]{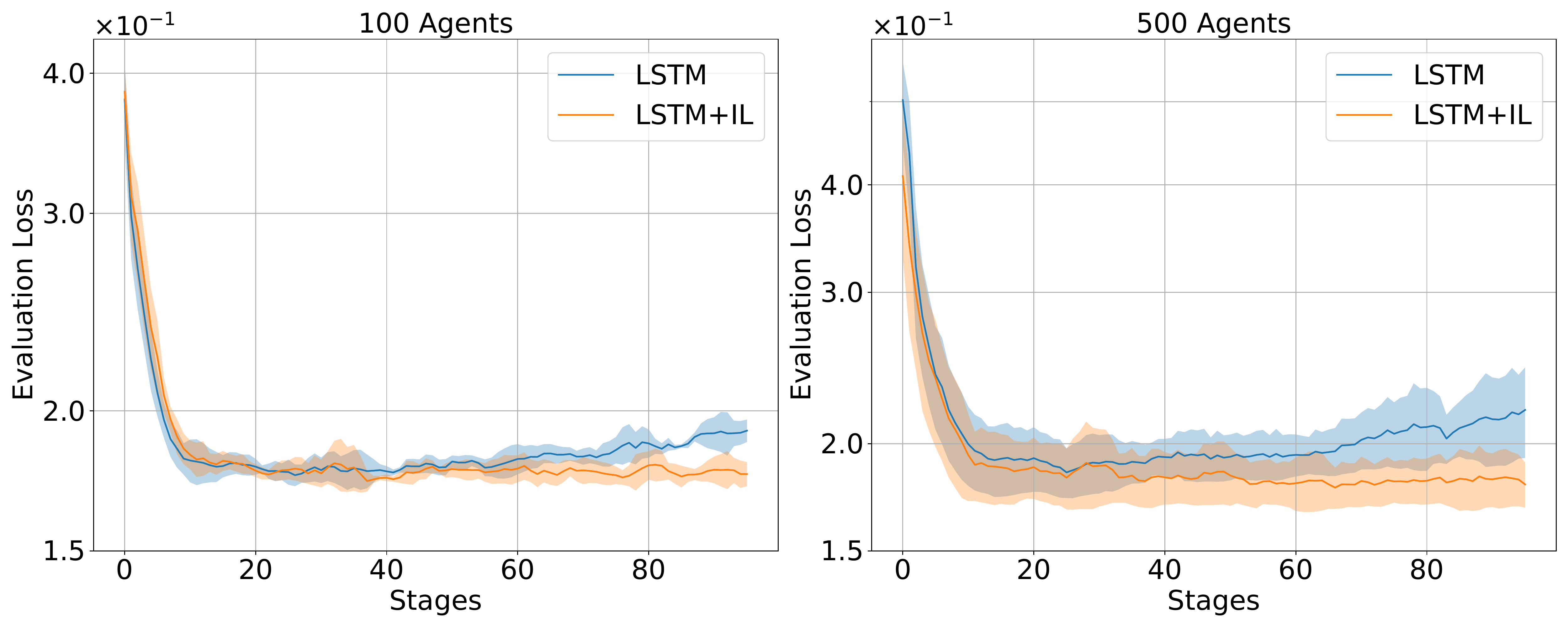}
  \vskip -0.2in
  \caption{
    Comparison between DFP-FBSDE w/ and w/o IL. The backbone is chosen as LSTM. The simulation is inter-bank game.
  }%
  \label{fig:invar}%
  \vskip -0.1in
\end{figure}
This is illustrated in fig.\ref{fig:invar} where the DFP-FBSDE model equipped with IL can handle the complex invariant representation, and the performance gap becomes even larger when the number of agents increase.

\subsection{Algorithm}
Following the analysis in $\S$\ref{sec:convergence} and $\S$\ref{sec:invar-layer}, here we introduce the SDFP-FBSDE algorithm which is scalable up to 3000 agents and has appreciable time and memory complexity by integrating IS and IL. We empirically verify the time and memory complexity shown in Fig.\ref{fig:complexity}.
\begin{figure}[t]
  \centering
  \vskip -0.1in
  \includegraphics[width=1.2\linewidth]{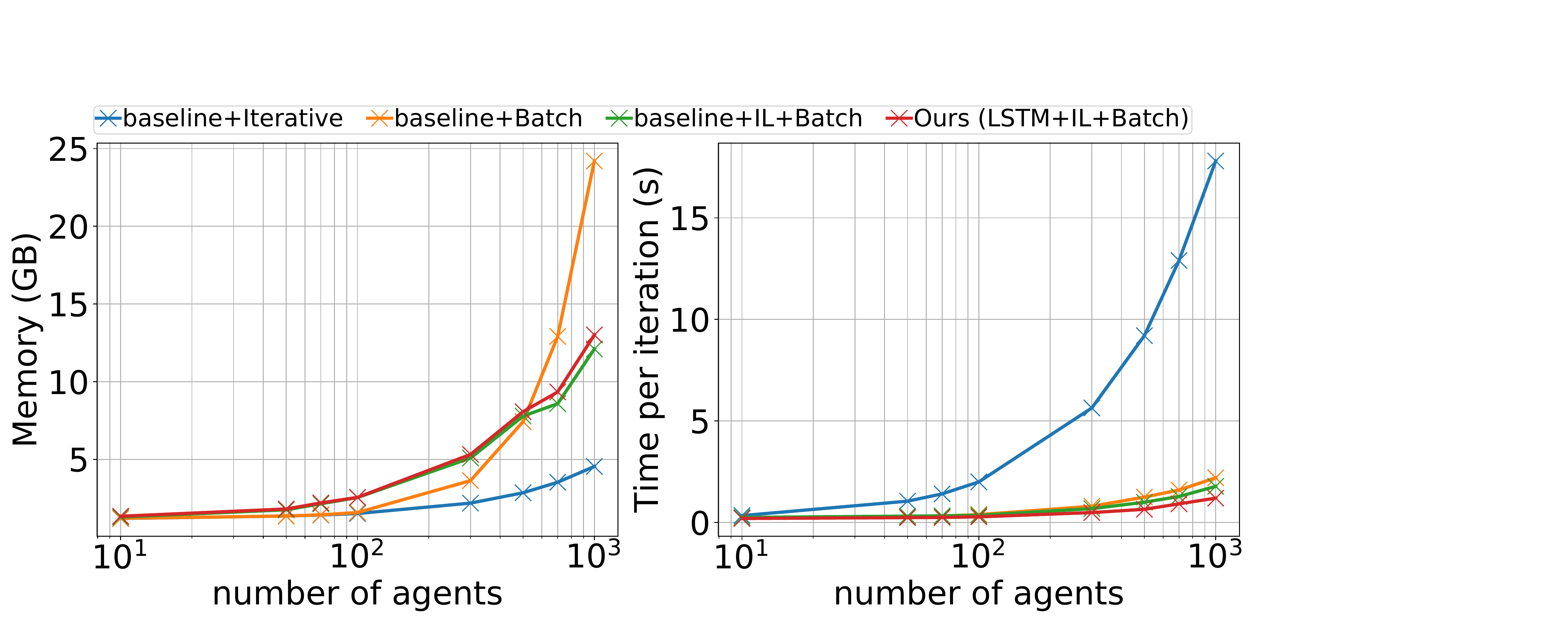}
  \vskip -0.2in
  \caption{
    Time and memory complexity comparison between batch, iterate and IL+batch implementations. Time complexity is measured by per-iteration time. IL stands for invariant layer.
  }%
  \label{fig:complexity}%
  \vskip -0.2in
\end{figure}

The neural network structure is same as shown in Fig.\ref{fig:architecture}. However, the feature extractor module (the blue boxes) will be replaced by an IL whose details can be found in Appendix \ref{Appendix:feature Extractor arch}. Meanwhile, the forward and backward process of $\mX$ and $V^i_t$ will be modified according to the Theorem \ref{thm:impt-samp}. Inspired by previous work \cite{theodorou2010stochastic}, the IS term ($K_s=\Gamma \Bar{\mU}^i$) term is selected as the control calculated from previous run of the algorithm. LSTM will be chosen as the backbone in our algorithm from experimental result. Notably, our algorithm will also improve the performance for the framework with FC backbone as used in \cite{han2019deep}, and the comparison will be elaborated in the simulation section. The full algorithm can be found in Appendix \ref{appendix:Algo}.

\section{Simulation}\label{sec:simulation}
In this section, we demonstrate the capability of SFDP-FBSDE on two different systems in simulation.
We first apply the framework to an inter-bank lending/borrowing problem, which is a classical multi-player non-cooperative game with an analytical solution. We compare against both the analytical solution and prior work \citep{han2019deep}. We also apply the framework to a variation of the problem for which no analytical solution exists. Finally, we showcase the general applicability of our framework in an autonomous racing problem in belief space and discuss how BSDE plays an importance role in a SDG. All experiment configurations can be found in Appendix \ref{appendix:config}.
we plot the results of 3 repeated runs with different seeds with the line and shaded region showing the mean and mean$\pm$standard deviation respectively. The hyperparameters and dynamics coefficients used in the inter-bank experiments are the same as \cite{han2019deep} unless otherwise noted. Technical differences between Baseline \cite{han2019deep} and SDFP-FBSDE are shown in Table.\ref{table:baseline-SDFP-diff}. The hardware used to run all simulations is included in Appendix \ref{Appendix:hardware}.
\begin{table}
	\label{table:baseline-SDFP-diff}
	\begin{sc}
	\begin{center}
	\begin{small}
	\begin{tabular}{r|r|r|r|cr}
	\toprule
	Algo & BackBone&Batch & IL & IS  \\
	\midrule
	Baseline    & FC+BN&\xmark&\xmark&\xmark\\
	SDFP & LSTM&  \checkmark&\checkmark&\checkmark\\
	\bottomrule
	\end{tabular}
	\end{small}
	\end{center}
	\end{sc}
	\vskip -0.1in
	\caption{Technical differences between baseline \cite{han2019deep} and SDFP. Batch stands for the batch scheme.}
	\vskip -0.15in
\end{table}

   
\subsection{Inter-bank lending/borrowing problem}
   \label{subsec:interbank}
We first consider an inter-bank lending/ borrowing model \citep{carmona2013mean} where the dynamics of the log-monetary reserves of $N$ banks is
  \begin{align*}
       \rd X_t^{i}&=\left[a(\Bar{X}_t-X_{t}^{i})+U_{t}^{i}\right]\rd t\\
       &+\sigma(\rho \rd W_{t}^{0}+\sqrt{1-\rho^{2}}\rd W_t^{i}),\\ 
       \Bar{X}_t&=\frac{1}{N}\sum_{i=1}^{N}X_t^{i}, i\in \mathbb{I}.
       \numberthis\label{fnc:inter-bank process}
  \end{align*}    
The state $X^i_t\in\mathbb{R}$ denotes the log-monetary reserve of bank $i$ at time $t>0$. The control $u_t^i$ denotes the cash flow to/from a central bank, where as $a(\bar{X}-X^i_t)$ denotes the lending/borrowing rate of bank $i$ from all other banks. The system is driven by $N$ independent standard Brownian motion $W_t^i$, which denotes the idiosyncratic noise, and a common noise $W_t^0$. The cost function has the form,
\begin{equation}\label{fnc:inter-bank running cost}
C^{i}_{t}(\mX,U_{i};\mU_{-i})=\frac{1}{2}U_i^{2}-qU_{i}(\bar{X}-X_i)+\frac{\epsilon}{2}(\bar{X}-X_i)^2.
\end{equation}
\begin{figure}[t]
  \centering
  \vskip 0.2in
  \includegraphics[width=1.25\linewidth]{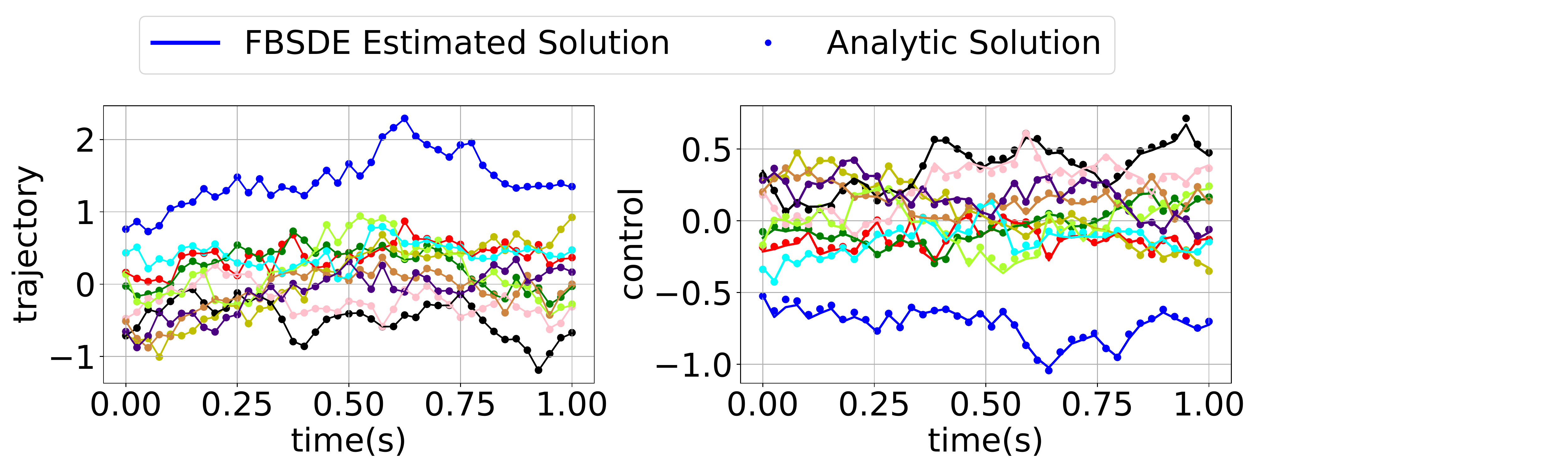}
  \vskip -0.2in
  \caption{
    Comparison of SDFP and analytical solution for the inter-bank problem. Both the state (\textit{left}) and control (\textit{right}) trajectories are aligned with the analytical solution (represented by dots).
  }%
  \label{fig:interbank10}%
  \vskip -0.15in
\end{figure}
The constants $q,\epsilon$ can be found in Appendix \ref{appendix:config}. The derivation of the FBSDEs and analytical solution are in Appendix \ref{Appendix:invar-layer}. We compare the result of our algorithm on a 10-agent problem with analytical solution and the baseline with the same hyperparameters.
Fig.~\ref{fig:interbank10}  shows the performance of SDFP-FBSDE compared with analytical solution. The state and control trajectories outputted by SDFP-FBSDE solution are aligned closely with the analytical solution.
\begin{table}[t]
\vspace{5 pt}
 \begin{center}
 \begin{tabular}{lll}
 \hline
 \multicolumn{1}{c}{\textsc{Framework}}  &\multicolumn{1}{c}{\textsc{RSE}} &\multicolumn{1}{c}{\textsc{Time (hr)}}
 \\ \hline
\textsc{Baseline}       &0.0423& 2.23\\
 \textsc{SFDP}             &0.0105& 1.55 \\
 \hline
 \end{tabular}
 \vspace{-5 pt}
 \caption{Comparison with \citet{han2019deep} on the 10-agent inter-bank problem.}
 \label{tab:10 agents compare}
 \end{center}
 \vspace{-15 pt}
 \end{table}
Table \ref{tab:10 agents compare} shows the numerical performance compared against prior work by \citet{han2019deep}. Our method outperforms by RSE (\ref{fnc:RSE-loss}) and computation wall time.

\textbf{Ablation Experiments:} In order to verify the effect of combination of IL and IS introduced in $\S$\ref{sec:convergence} and $\S$\ref{sec:invar-layer}, we perform the ablation experiments on the 500 agents inter-bank problem. Fig.\ref{fig:ablation-exp} shows the model equipped with IS and IL would obtain superior convergence rate and better results on evaluation set regardless of the choice of backbone.

\begin{figure}[t]
  \centering
  \vskip 0.2in
  \includegraphics[width=1.0\linewidth]{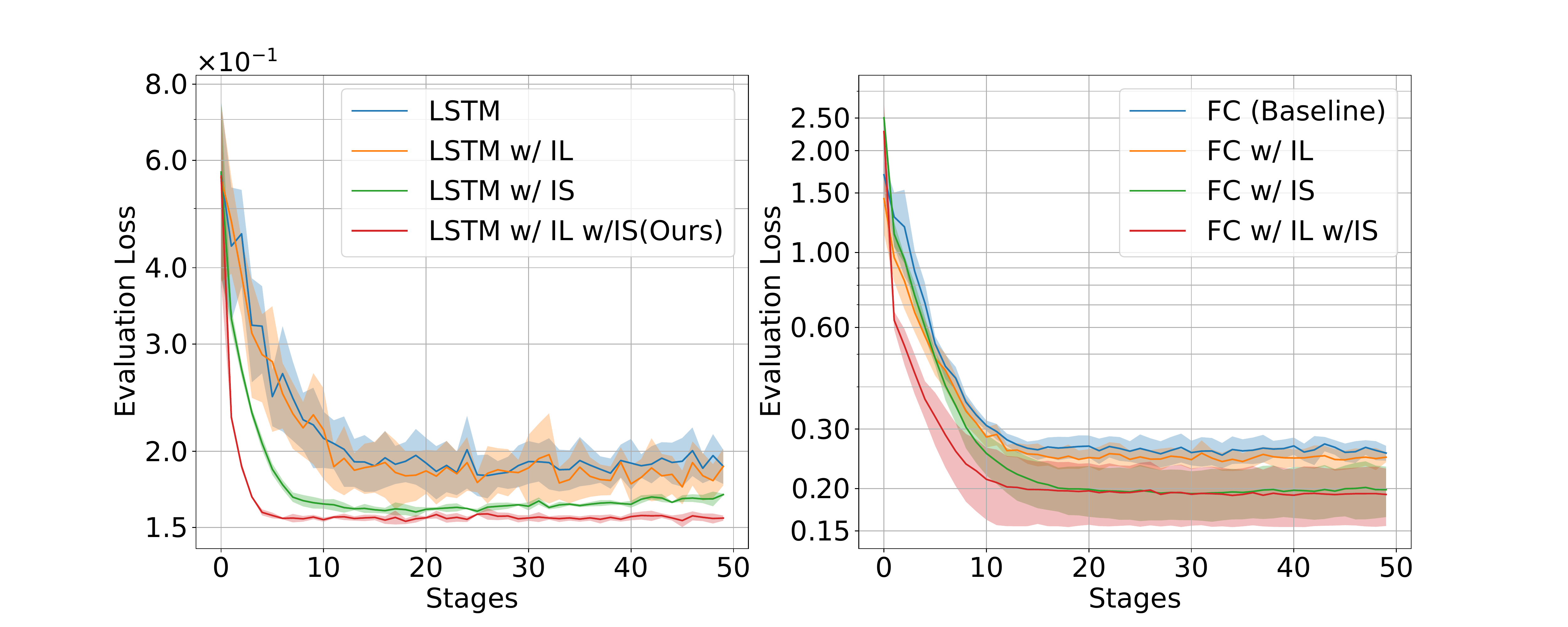}
  \vskip -0.2in
  \caption{
    Ablation experiments on LSTM and FC backbone.
  }%
  \label{fig:ablation-exp}
  \vskip -0.1in
\end{figure}
\begin{figure}[t]
  \centering
  \vskip 0.2in
  \includegraphics[width=1.0\linewidth]{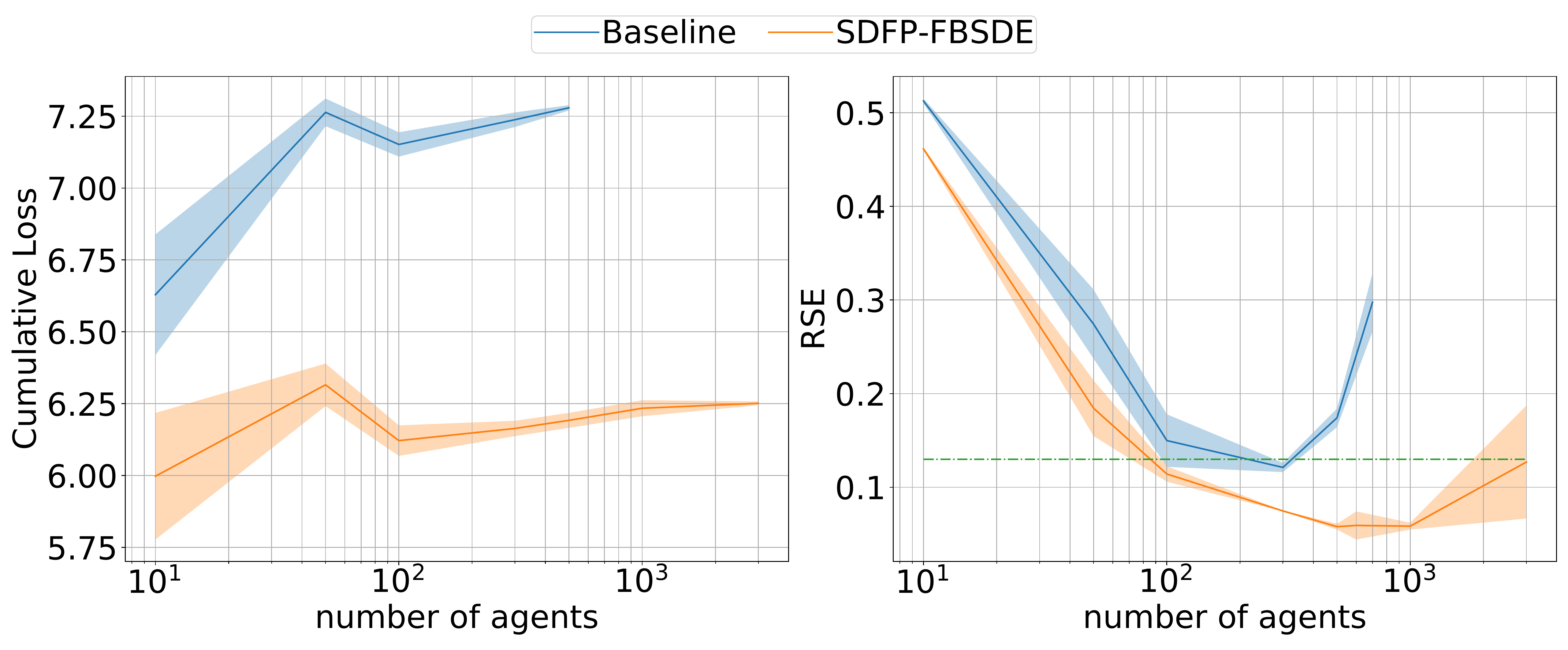}
  \vskip -0.2in
  \caption{
    Comparison of SDFP-FBSDE and Baseline for inter-bank problem with different number of agents evaluated on cumulative loss(\ref{fnc:acc-loss}) and RSE(\ref{fnc:RSE-loss}).
  }%
  \label{fig:RSE-over-agents}
  \vskip -0.15in
\end{figure}
\textbf{High Dimension Experiments:} Convinced by the ablation experiments, we conducted comparison between baseline and our proposed algorithm on different numbers of agents (up to 3000). Because of the limitation of graphical memory of hardware, the maximum number of agents for baseline is 700. The result is demonstrated in Fig.\ref{fig:RSE-over-agents}. SDFP-FBSDE outperforms the baseline by cumulative total cost (\ref{fnc:acc-loss}) and RSE (\ref{fnc:RSE-loss}). As been shown in FBSDE theory (\ref{fnc:FBSDE}), the augmented control is characterized by $\mU_{0,*}$. This means more agents will lead to more actuator and better explorations hence the RSE loss would decrease when number of agent is smaller than 300. When the number of agents keeps growing, the network would have difficulty of learning many agents representation, thus leading to the increase of RSE. Our method mitigates the curse of many agent, and postpones the RSE turning point to 1000 agents. Notably, SDFP-FBSDE has only marginally larger RSE loss at 3000 agents compared with baseline at 300 agents. The left subplot in Fig.\ref{fig:RSE-over-agents} shows our algorithm found more reasonable control sequence than baseline. The margin between SDFP-FBSDE and baseline is even larger on evaluation loss (\ref{fnc:eval-loss}). Since this plot is not informative because of the large difference of scale, we defer it to Appendix \ref{appendix:eval-loss-inter-bank}.

\textbf{Superlinear Experiments}:
We also consider a variant of dynamics in equation (\ref{fnc:inter-bank process}):
 \begin{align*}
   \rd X_t^{i}&=\left[a(\Bar{X}-X_{t}^{i})^{3}+U_{t}^{i}\right]\rd t\\&+\sigma(\rho \rd W_{t}^{0}+\sqrt{1-\rho^{2}}\rd W_t^{i}),\\ \Bar{X}_t&=\frac{1}{N}\sum_{i=1}^{N}X_t^{i}, i\in \mathbb{I}.\numberthis
 \end{align*}
Due to the nonlinearity in the drift term, analytical solution or simple numerical representation of the Nash equilibrium does not exist \citep{han2019deep}. 
The drift rate $a$ is set to $1.0$ to compensate for the vanishing drift term caused by super-linearity. Heuristically, the distribution of control and state should be more concentrated than that of the linear dynamics. We compare the state and control of agent $i$ at terminal time against analytical solution and Deep FBSDE solution of the linear dynamics with the same coefficients. Fig.~\ref{fig:500agent-distribution} is generated by evaluating the trained Deep FBSDE model with a batch size of 50000. Fig.\ref{fig:500agent-distribution} shows that the solution from superlinear dynamics is more concentrated as expected. The terminal control distribution verifies that the superlinear drift term pushes the state back to the average faster than linear dynamics and thus requires less control effort. Since the numerical solution is not available, we compare the cumulative loss (\ref{fnc:acc-loss}) loss and evaluation loss (\ref{fnc:eval-loss}) between baseline and our algorithm (see Appendix \ref{Appendix:superlinear}).
\begin{figure}[t]
  \centering
  \vskip 0.2in
  \includegraphics[width=1.03\linewidth]{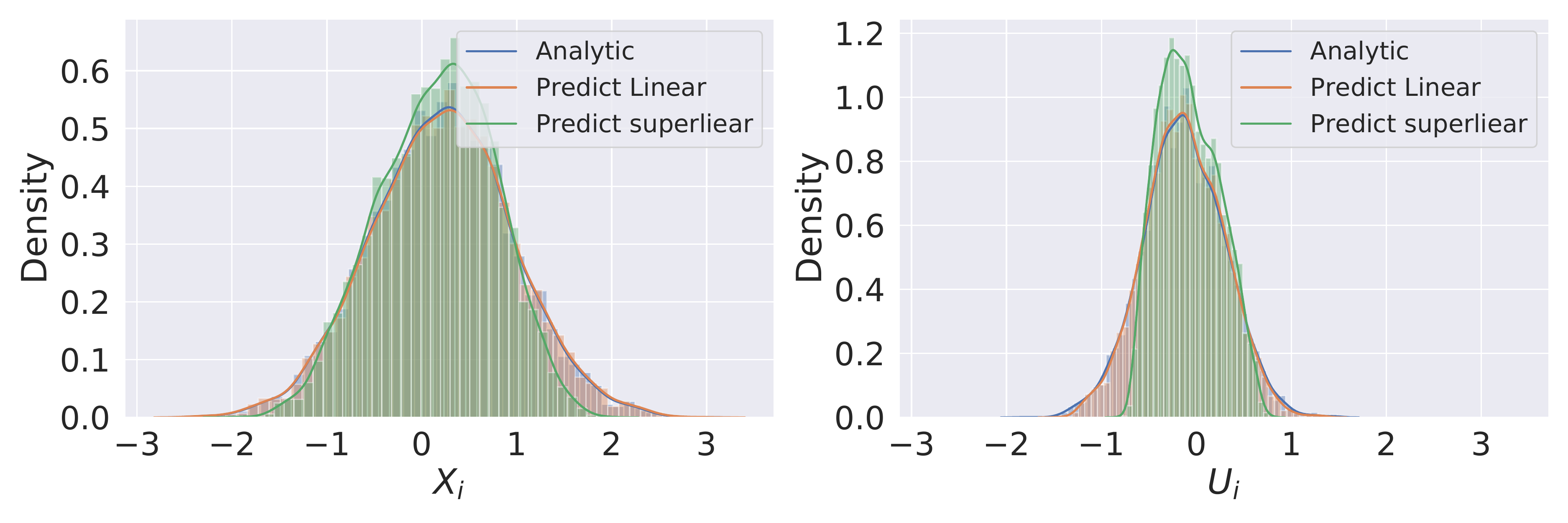}
  \vskip -0.2in
  \caption{
    Terminal time step state $\mX$ and control $\mU$ distribution of $i$th agent for linear and superlinear dynamics. The analytic and predicted distribution for linear case are very similar.
  }%
  \label{fig:500agent-distribution}
  \vskip -0.2in
\end{figure}
\vspace{-0.2in}
\subsection{Extension to Partially Observed Game in Robotics}
\label{sec:car racing result}
In this section, we demonstrate the general applicability of our framework on an autonomous racing example in belief space, and show how BSDE influences the game when competition is  triggered experimentally. we consider a autonomous racing problem with race car dynamics
\begin{equation}
       \dot{X}_i = [v\cos\theta, v\sin\theta, u_{\text{acc}}-c_{\text{drag}}v, u_{\text{steer}}v/L]^\mathrm{T}
\end{equation}
where $X_i = [x,y,v,\theta]\T$ represent the $x, y$ position, forward velocity and heading of the $i$th vehicle. Here we assume $x,y,v,u_{\text{acc}}\in\mathbb{R}, u_{\text{steer}}\in[-1,1]$. The goal of each player is to drive faster than the opponent, stay on the track and avoid collision. The running cost $C^i$ is defined in Appendix \ref{appendix:car-racing cost}.
During the game, players have access to partially observed global augmented states estimated by an Extended Kalman Filter (EKF) and opponent's controller in the past games. Additionally, we assume that stochasticity enters the system through the control channels and have a continuous-time noisy observation model with full-state observation. The FBSDE derivation of belief space stochastic dynamics is included in Appendix \ref{appendix:racing}. We focus on the 2-car racing scenario for the interpretability of results, since no analytical solution exists for the problem. With the 2-car setup, interesting behaviors arise when the cost function of each car is altered. Since all the trials complete 1 lap, here we only show the first 8 seconds' result for neatness.

\textbf{Non-cooperative and Non-competitive Case} We first test the capability of our framework on a partially observed learning problem where all states are estimated with observation noise and additive noise in the system. Both cars can stay in the track as expected. Since there is no competition between the two cars, they demonstrate similar behaviors. The plot of trajectories with 64 trails of games can be found on the top left of Fig.\ref{fig:Car-racing}. The EKF posterior estimation is not drawn in Fig.\ref{fig:Car-racing} for neatness. For single trails plot with EKF posterior can be found in Appendix \ref{Appendix:car-racing post}.

\begin{figure}[t]
  \centering
  \vskip -0.1in
\includegraphics[width=1.1\linewidth]{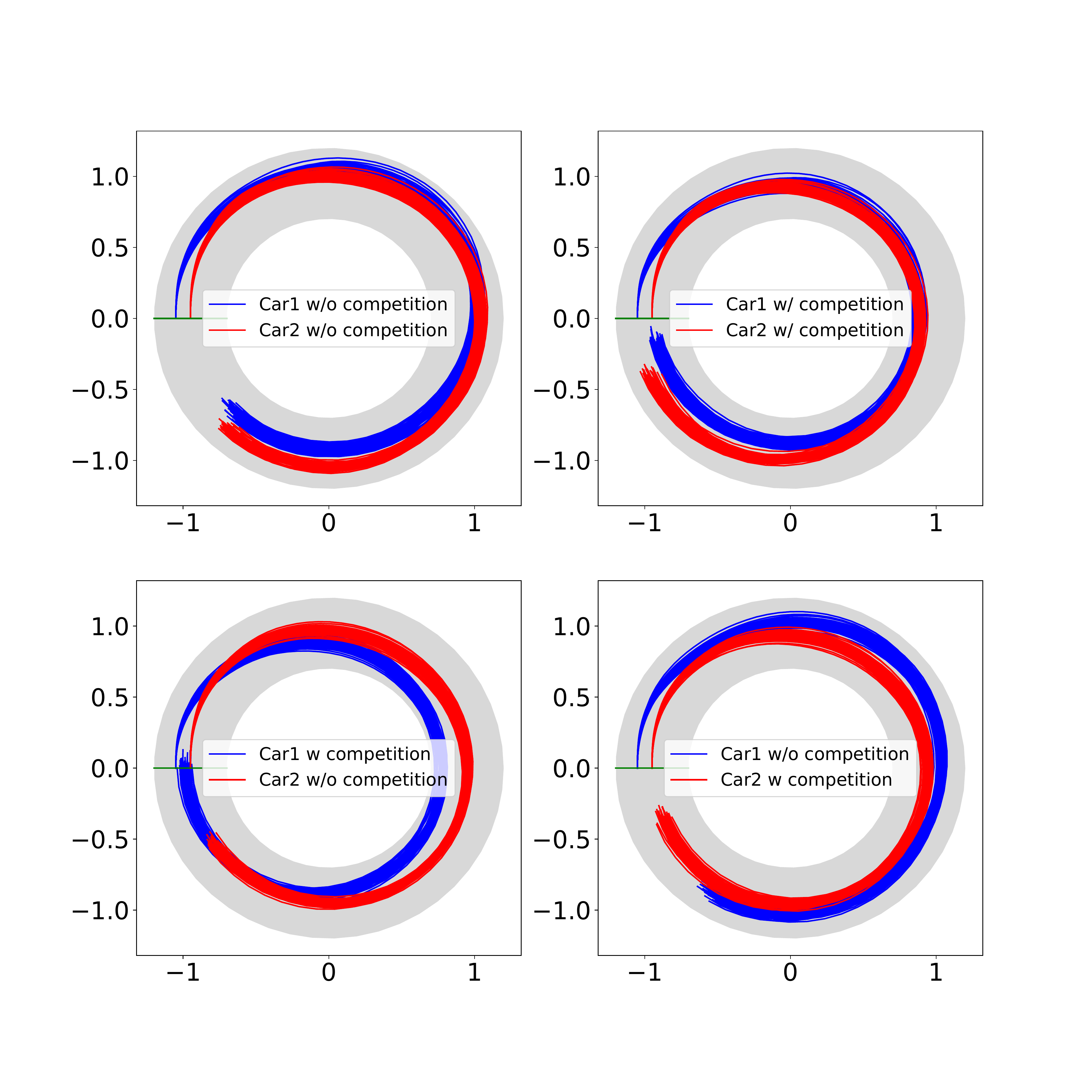}
  \vskip -0.4in
  \caption{
        The plots contains 64 trials of racing. The performance varies with respect to the competition loss.
  }%
  \label{fig:Car-racing}
  \vskip -0.1in
\end{figure}

\textbf{Competitive Case } When we add competition loss (see Appendix \ref{appendix:car-racing cost}) on the cars, the propagation of BSDE (\ref{fnc:FBSDE}) will be modified accordingly. Therefore, the racing game demonstrates interesting properties. When competition loss is applied on both cars, both of them try to cut the corner in order to occupy the leading position as shown on the top right of Fig. \ref{fig:Car-racing}, and both of them travel longer distance compared with the ones do not have competitive loss. When competition loss is present in only one of the two cars, then the one with competition loss dominates the game as shown in the botton subplots of Fig.\ref{fig:Car-racing}. 

\vspace{-3pt}
\section{Conclusion}
\vspace{-5pt}
\label{sec:conclusion}
In this paper, we first extend the theoretical analysis from \cite{han2020convergence} and introduce importance sampling to improve the sample efficiency and convergence rate. To further push our work to handle larger number of agents with appreciable time and memory complexity, batch query scheme and invariant layer implementation are proposed. The scalability of our algorithm, along with a detailed sensitivity analysis, is demonstrated in an inter-bank borrowing/lending example. Our framework achieves better performance in different metrics and scales to significantly higher dimensions than the state-of-the-art. The general applicability of our framework is showcased on a belief space racing problem in the partially observed scenario.

\section*{Acknowledgement}
This project is supported under ARO W911NF2010151. We would also like to thank Professor Matthieu Bloch and Guan-Horng Liu for the helpful discussions.

\nocite{langley00}

\bibliography{example_paper}
\bibliographystyle{icml2021}

\newpage
\appendix
\onecolumn
\begin{center}
\textbf{{\huge Supplementary Material}}
\end{center}
\section{Multi-agent HJB Derivation}\label{Appendix:hjb_derivation}
Applying Bellman's principle to the value function (\ref{Value_function}) yields,
\begin{equation}\label{networked-functional}
\begin{aligned}
    V^i(t,\mX(t))&=\inf\limits_{U_i \in \mathcal{U}_i}\E\left[V^i(t+\rd t,\mX(t+dt))+\int_t^{t+dt} C^{i}\rd\tau\right]\\
    &=\inf \limits_{U_i \in \mathcal{U}_i}\E\left[C^{i}\rd t +V^i(t,\mX(t))+V_t^i(t,\mX(t))\rd t\right.\\
    &+\left. V_x^{i\mathrm{T}}(t,\mX(t))\rd\mX+\frac{1}{2}\tr(V_{xx}(t,\mX(t)\Sigma\Sigma\T) \rd t\right]\\
    &=\inf \limits_{U_i \in \mathcal{U}_i}\E\left[C^{i}\rd t +V^i(t,\mX(t))+V_t^i(t,\mX(t))\rd t\right.\\
    &+\left. V_x^{i\mathrm{T}}(t,\mX(t))((f+G\mU)\rd t+\Sigma\rd\boldsymbol{W})+\frac{1}{2}\tr(V_{xx}^{i}(t,\mX(t))\Sigma\Sigma\T)\rd t\right]\\
    &= \inf \limits_{U_i \in \mathcal{U}_i}\left[C^{i}\rd t +V^i(t,\mX(t))+V_t^i(t,\mX(t))\rd t\right.\\
    &+\left. V_x^{i\mathrm{T}}(t,\mX(t))((f+G\mU)\rd t)+\frac{1}{2}\tr(V_{xx}^{i}(t,\mX(t))\Sigma\Sigma\T)\rd t\right]\\
    \Rightarrow 0 &= V_t^i(t,\mX(t)) + \inf \limits_{U_i \in \mathcal{U}_i}\left[C^{i}+ V_x^{i\mathrm{T}}(t,\mX(t))(f+G\mU)\right]+\frac{1}{2}\tr(V_{xx}^{i}(t,\mX(t))\Sigma\Sigma\T)
\end{aligned}    
\end{equation}
Given the cost function assumption (the cost function is quadratic w.r.t control variable.), the infimum can be obtained explicitly using optimal control $U_{i}^*=-R^{-1}(G_i^\mathrm{T}V^i_x+Q_i^\mathrm{T}\mX)$. With that we can obtain the final form of the HJB PDE as
\begin{equation}\label{Appendix:HJB}
    V^i_t+h+V^{i\mathrm{T}}_x (f+G \mU_{0,*})+\frac{1}{2}\tr(V^i_{xx}\Sigma\Sigma\T)=0, \quad V^i(T,\mX)=g(\mX(T)).
\end{equation}

\section{Multi-agent FBSDE Derivation}\label{Appendix:FBSDE_derivation}
Given the HJB PDE in \eqref{Final_form_HJB}, one can apply the non-linear Feynman-Kac lemma \cite{karatzas1991brownian} to obtain a set of FBSDE as
\begin{equation}
   \begin{aligned}
		\rd  \mX(t)&=(f +G\mU_{0,*})\rd t+\Sigma\rd\boldsymbol{W},\quad \mX(0)=\vx_0 \quad (\text{FSDE})\\
        \rd V^i&=-h\rd t+V_{x}^{i\mathrm{T}} \Sigma \rd \boldsymbol{W}, \quad V(\mX(T))=g(\mX(T)).  \quad (\text{BSDE})
	\end{aligned}    
\end{equation}

The backward process is derived by applying Ito's lemma on $V^i$
\begin{align*}
\rd V^i &= V^i_t \rd t + V_x^{i\mathrm{T}}\rd \mX + \frac{1}{2}\tr(V^i_{xx}\Sigma\Sigma^\mathrm{T})\rd t\\
&\text{Plug in eq.\ref{Appendix:HJB} into} V_t^i \text{ term}, \text{and eq.\ref{state_process}} \text{into } \rd \mX \text{ term}\\
&= (-h-V^{i\mathrm{T}}_x (f+G \mU_{0,*})-\frac{1}{2}\tr(V^i_{xx}\Sigma\Sigma\T))\rd t + V_x^{i\mathrm{T}}((f +G\mU_{0,*})\rd t+\Sigma\rd\boldsymbol{W}) + \frac{1}{2}\tr(V^i_{xx}\Sigma\Sigma\T)\rd t\\
&= -h\rd t + V_x^{i\mathrm{T}}\Sigma\rd \boldsymbol{W}.
\end{align*}
\section{Metrics}\label{appendix:metrics}
In the test set, we randomly select $B$ initial states, and $B\times T$ noise $\mW$, where $B$ is Batch size and $T$ is time horizon. We evaluate the performance of models based on three different metrics. All losses (for comparison) are computed by averaging the last 10 stages and over 3 random seeds for fair comparison.
\subsection{Relative Square Error(RSE)}\label{sec:RSE-loss}
The RSE is a metric applied in test phase which is defined as following:
\begin{equation}\label{fnc:RSE-loss}
\begin{aligned}
\mathcal{L}_{RSE}=\frac{\sum{\substack{i\in \sI \\ 1\leq j \leq B}}(\hat{Y}^i(0,\mX^{j}(0))-Y^i(0,\mX^{j}(0)))^2}{\sum{\substack{i\in \sI \\ 1\leq j \leq B}}(\hat{Y}^i(0,\mX^{j}(0))-\bar{Y}^i(0,\mX^{j}(0)))^2},
\end{aligned}    
\end{equation} 
Where $Y^i$ is the analytical solution of value function for $i$th agents at initial state $\mX^j(0)$. The initial state $\mX^{j}(0)$ is new batch of data sampled from same distribution as $\mX(0)$ in the training phase. The batch size $B$ is 256 for all inter-bank simulations. $\hat{Y}^i$ is the approximated value function for $i$th agent by FBSDE controller, and $\bar{Y}^i$ is the average of analytical solution for $i$th agent over the entire batch.
\subsection{Evaluation/training Loss}\label{sec:eval-loss}
The evaluation loss is same as training loss which is defined as the mean square error between true terminal value evaluated on the terminal state and the value propagated by the BSDE,
\begin{equation}
\begin{aligned}\label{fnc:eval-loss}
\mathcal{L}(\hat{Y}_T^i,Y_T^i)=\frac{1}{B}||\hat{Y}_T^i-Y_T^i||_2^2
\end{aligned}    
\end{equation} 
\subsection{Cumulative Loss}\label{sec:acc-loss}
The cumulative loss is computed explicitly from the objective function of optimal control (\ref{cumulative_loss}) in the test phase. Here we restate it for completeness.
 \begin{equation}\begin{split}\label{fnc:acc-loss}
	\mathcal{L}_{cum}&:= J^i_t(\mX,U_{i,m};\mU_{-i,m-1})\\
	&=\E\left[g(\mX_T)+\int_0^T C^i(\mX_\tau,U_{i}(\mX_\tau);\mU_{-i})\rd \tau\right],
	\end{split}
\end{equation}

\section{Missing Derivation and Proof in Section \ref{sec:SDFP-FBSDE}}
\subsection{Assumption 2}\label{appendix:liptz}
Here we state the assumption \ref{asmpt:liptz} in detail. 

Consider a general FBSDE system,
\begin{equation}
\begin{aligned}
    &\mX_T^{t,\vx}=\vx+\int_{t}^{T}\mu_s\rd s+\int_{t}^{T}\Sigma_s\rd\mW_s &\text{ (FSDE)},\\ 
    & Y_{t}^{T,\vx}=g(\mX_T^{t,\vx})-\int_{t}^{T}H_s\rd s+\int_{t}^{T}Z_s \rd \mW_s &\text{ (BSDE)}, \\
\end{aligned}    
\end{equation}
We use $|\cdot|$ and $||\cdot||_F$ to denote the $L^2$ norm and Frobenius norm respectively. For terminal loss $g(\cdot)$, drift function $\mu(\cdot,\cdot,\cdot)$, $H(\cdot,\cdot,\cdot)$, diffusion function $\Sigma(\cdot,\cdot)$, and control function $\vu(\cdot,\cdot)$, they are Lipschitz continuous with respect to their arguments:
\begin{equation}
	\begin{aligned}
	    |g(t,\vx_1)-g(t,\vx_2)|^2&\leq g_x|\vx_1-\vx_2|^2\\
	    |H(t,\vx_1,\vz_1)-H(t,\vx_2,\vz_2)|^2&\leq H_x|\vx_1-\vx_2|^2+H_{z}||\vz_1-\vz_2||_F^2,\\
		||\Sigma(t,\vx_1)-\Sigma(t,\vx_2)||_F^2&\leq\Sigma_x|\vx_1-\vx_2|^2\\
		||\vu(t,\vx_1)-\vu(t,\vx_2)||^2&\leq u_x|\vx_1-\vx_2|^2\\
		|\mu(t,\vx_1,\vu_1)-\mu(t,\vx_2,\vu_2)|^2&\leq \mu_x|\vx_1-\vx_2|^2+\mu_{u}|\vu_1-\vu_2|^2,\\
	\end{aligned}
\end{equation}	
Here $\mu_x,\mu_u,g_x,H_x,H_z,\Sigma_x,u_x$ are all positive constants.
\subsection{Proof of Lemma \ref{Lem: Y Bound}}\label{lem:lemma1-proof}
\begin{proof}
Denote $(\mX_s^{t,\vx},Y_s^{t,\vx},Z_s^{t,\vx})_{t\leq s\leq T}$ as the solution for the FBSDE system for the $i$th agent:
\begin{equation}
\begin{aligned}
    &\mX_T^{t,\vx}=\vx+\int_{t}^{T}\mu_s\rd s+\int_{t}^{T}\Sigma_s\rd\mW_s &\text{ (FSDE)},\\ 
    & Y_{t}^{T,\vx}=g(\mX_T^{t,\vx})-\int_{t}^{T}H_s\rd s+\int_{t}^{T}Z_s \rd \mW_s &\text{ (BSDE)}, \\
\end{aligned}    
\end{equation}
for any $(t_0,\vx)\in [t_0,T]\times \mathcal{X}$. Then for any $t\in [t_0,T]$ and $\vx_1,\vx_2 \in \mathcal{X}$, let $(\mX_t^{j},Y_t^{j},Z_t^{j})$ be the short notation of $(\mX_t^{t_0,\vx_j},Y_t^{t_0,\vx_j},Z_t^{t_0,\vx_j})$, where $j\in\left\{1,2\right\}$, $t\in[0,T]$.\\
Here we define $\delta \mX_t,\delta Y_t,\delta Z_t,\delta \Sigma_t,\delta H_t$ as:
\begin{equation}
	\begin{aligned}
		\delta \mX_t&=\mX_t^{1}-\mX_t^{2},\\
		\delta Y_t&=Y_t^{1}-Y_t^{2},\\
		\delta Z_t&=Z_t^{1}-Z_t^{2},\\
		\delta \mu_t&=\mu_t(t,\mX_t^{1},\mU_t^1)-\mu(t,\mX_t^{2},\mU_t^2),\\
		\delta H_t&=H(t,\mX_t^{1},Z_t^{1})-H(t,\mX_t^{2},Z_t^{2}),\\
		\delta \Sigma_t&=\Sigma(t,\mX_t^{1})-\Sigma(t,\mX_t^{2}),\\
	\end{aligned}
\end{equation}
Then we have:
\begin{equation}
	\begin{aligned}
		\rd\delta \mX_t&=\rd\mX_t^{1}-\rd\mX_t^{2}\\
		&=(\mu_t(t,\mX_t^1,\mU_t^{1})\rd t+\Sigma(t,\mX_{t}^{1})\rd W_t)-(\mu_t(t,\mX_t^{2},\mU_t^{2})\rd t+\Sigma(t,\mX_{t}^{2})\rd W_t)\\
		&=(\delta \mu_t)\rd t+\delta\Sigma_t \rd W_t,\\
		\rd\delta Y_t&=-\delta H_t\rd t+\delta Z_t\T\rd W_t.\\
	\end{aligned}
\end{equation}
By applying It\^{o} lemma to $\rd \delta \mX_t$ and $\rd \delta Y_t$,
\begin{equation}
	\begin{aligned}
		\rd |\delta \mX_t|^2&=(2\delta \mu_t\delta \mX_t+\frac{1}{2}\cdot 2 \delta ||\Sigma_t||_F^2)\rd t+2(\delta \mX_t)\T\delta\Sigma_t\rd W_t\\
		&=(2\delta \mu_t\delta \mX_t+\delta ||\Sigma_t||_F^2)\rd t+2(\delta \mX_t)\T\delta\Sigma_t\rd W_t\\
		\rd |\delta Y_t|^2&=(-2\delta H_t\delta Y_t+\delta ||Z_t\T||_F^2)\rd t+2(\delta Z_t\delta Y_t)\T\rd W_t\\
	\end{aligned}
\end{equation}
By taking the expectation on both sides of $\rd|\delta \mX_t|^2$, it will yield:
\begin{equation}
	\begin{aligned}
		\mathbb{E}[|\delta \mX_t|^2]&=|\vx_1-\vx_2|^2+\int_{t_0}^{t}\mathbb{E}\left[2\delta \mu_s\delta \mX_s+\delta ||\Sigma_s||_F^2\right]\rd s\\
        &\leq |\vx_1-\vx_2|^2+\int_{t_0}^{t}\mathbb{E}\left[(\mu_x+\mu_u u_x)^{-1}|\delta \mu_s|^2+(\mu_x+\mu_u u_x)|\delta \mX_s|^2+\delta ||\Sigma_s||_F^2\right]\rd s\\
        &\leq |\vx_1-\vx_2|^2+\int_{t_0}^{t}\mathbb{E}\left[(\mu_x+\mu_u u_x)^{-1}(\mu_x|\delta \mX_s|^2+\mu_u |\delta \vu|^2 )+(\mu_x+\mu_u u_x)|\delta \mX_s|^2+\delta ||\Sigma_s||_F^2\right]\rd s\\
		&\leq |\vx_1-\vx_2|^2+\int_{t_0}^{t}\mathbb{E}\left[(\mu_x+\mu_u u_x)^{-1}(\mu_x|\delta \mX_s|^2+\mu_u u_x|\delta \mX_s|^2)+(\mu_x+\mu_u u_x)|\delta \mX_s|^2+\delta ||\Sigma_s||_F^2\right]\rd s\\
		&\leq |\vx_1-\vx_2|^2+\int_{t_0}^{t}\mathbb{E}\left[|\delta \mX_s|^2+(\mu_x+\mu_u u_x)|\delta \mX_s|^2+\Sigma_x|\delta \mX_s|^2\right]\rd s\\
		&= |\vx_1-\vx_2|^2+(I+\mu_x+\mu_u u_x+\Sigma_x)\int_{t_0}^{t}\mathbb{E}|\delta \mX_s|^2\rd s\\
		&\text{by Gronwall's inequality}\\
		&\leq e^{(I+\mu_x+\mu_u u_x+\Sigma_x)(t-t_0)}|\vx_1-\vx_2|^2\\
		&= e^{\xi(t-t_0)}|\vx_1-\vx_2|^2
	\end{aligned}
\end{equation}
Where $\xi=I+\mu_x+\mu_u u_x+\Sigma_x$.

Similarly, we can have,
\begin{equation}
	\begin{aligned}
		\mathbb{E}[|\delta \mY_t|^2]&=\mathbb{E}|\delta \mY_T|^2+\int_{t}^{T}\mathbb{E}\left[2\delta H_s\delta Y_s-\delta ||\mZ_s\T||_F^2\right]\rd s\\
		&=\mathbb{E}|\vg(T,\vx_1)-\vg(T,\vx_2)|^2+\int_{t}^{T}\mathbb{E}\left[2\delta H_s\delta Y_s\right]-\mathbb{E}\||\mZ_s||_F^2\rd s\\
		&\leq g_x\mathbb{E}|\delta \mX_T|^2+\int_{t}^{T} H_z \mathbb{E}|\delta Y_s|^2+ H_z^{-1} \mathbb{E}|\delta H_s|^2
		-\mathbb{E}||\mZ_s||_F^2\rd s\\
		&\leq g_x\mathbb{E}|\delta \mX_T|^2+\int_{t}^{T} H_z \mathbb{E}|\delta Y_s|^2+ H_z^{-1} \mathbb{E}\left[H_x|\delta \mX_s|^2+H_z |\delta \mZ_s|^2\right]-\mathbb{E}||\mZ_s||_F^2\rd s\\
		&= g_x\mathbb{E}|\delta \mX_T|^2+\int_{t}^{T} H_z \mathbb{E}|\delta Y_s|^2+ H_z^{-1}H_x \mathbb{E}|\delta \mX_s|^2 \rd s\\
		&\leq \left[ g_x e^{\xi(t-t_0)}+H_x\frac{e^{\xi(T-t)}-e^{\xi(t-t_0)}}{H_z \Sigma_x}\right]|\vx_1-\vx_2|^2+H_z\int_{t}^{T} \mathbb{E}|\delta Y_s|^2 \rd s\\
		&\text{by Gronwall's inequality}\\
		|\delta \mY_t|^2&\leq e^{H_z(T-t)}\left[ g_x e^{\xi(T-t_0)}+H_x\frac{e^{\xi(T-t_0)}-e^{\xi(t-t_0)}}{H_z \Sigma_x}\right]|\vx_1-\vx_2|^2\\
	\end{aligned}
\end{equation}
When $t_0=0$, one can have:
\begin{equation}
	\begin{aligned}
    |\delta Y_T|^2&\leq g_x e^{\xi T}|\vx_1-\vx_2|^2\\
    &=L_1 |\vx_1-\vx_2|^2\\
    |\delta Y_0|^2&\leq e^{H_zT}\left[ g_x e^{\xi T}+H_x\frac{e^{\xi T}-1}{H_z \Sigma_x}\right]|\vx_1-\vx_2|^2\\
    &=L_2 |\vx_1-\vx_2|^2\\
	\end{aligned}
\end{equation}
Where
\begin{equation}
	\begin{aligned}
        L_1&=g_x e^{\xi T}\\
        L_2&=e^{H_zT}\left[ g_x e^{\xi T}+H_x\frac{e^{\xi T}-1}{H_z \Sigma_x}\right]\\
        \xi&=I+\mu_x+\mu_u u_x+\Sigma_x
	\end{aligned}
\end{equation}
\end{proof}

\subsection{Lemma.2 with Proof}\label{Lem: proof-lemma1-impt-samp}
\begin{lemma}
	\label{Lem:lemma1-impt-samp}
    Denote $(\mX_s^{t,\vx},Y_s^{t,\vx},Z_s^{t,\vx})_{t\leq s\leq T}$ as the solution for the FBSDE system with importance sampling (\ref{fnc:impt-FSDE}, \ref{fnc:impt-BSDE}) satisfying assumptions \ref{asmpt:dyn} and \ref{asmpt:liptz}. Denote the difference of $Y$ component at two different states $\vx_1$ and $\vx_2$ as:
	\begin{equation}
		\begin{aligned}
		    \delta \mX_t=\mX_t^{t_0,\vx_1}-\mX_t^{t_0,\vx_2}, \delta Y_t=Y_t^{t_0,\vx_1}-Y_t^{t_0,\vx_2}.
		\end{aligned}
	\end{equation}	
	Then we can have:
	\begin{equation}
		\begin{aligned}
			|\delta Y_{T}|^2 &\leq \tilde{L}_1|\vx_1-\vx_2|^2, \\
			|\delta Y_{t_0}|^2 &\leq \tilde{L}_2|\vx_1-\vx_2|^2,
		\end{aligned}
	\end{equation}	
	Where $L_1$ and $L_2$ are defined as:
	\begin{equation}
		\begin{aligned}
		    \tilde{L}_1&=g_xe^{\tilde{\xi}}\\
			\tilde{L}_2&= e^{2(H_z+k_z)(T-t_0)}\left[g_x e^{\tilde{\xi}(T-t_0)}+H_x(H_z^{-1}+k_z^{-1})\frac{e^{\tilde{\xi}(T-t_0)}-1}{2 \Sigma_x}\right],\\
			&\tilde{\xi}=I+\tilde{\mu}_x+\tilde{\mu}_u u_x+\Sigma_x,
		\end{aligned}
	\end{equation}
	Where $\mu_x,\mu_u,\Sigma_x,H_x,H_z, g_x,u_x$ are Lipschitz constant defined in Assumption.\ref{asmpt:liptz}. The definition of Lipschitz constant $\tilde{\mu}_x,\tilde{\mu}_u,k_z$ and proof can be found in the following proof.
\end{lemma}
\begin{proof}
Similar to the proof of lemma.\ref{Lem: Y Bound}, now we first analyze the forward process with IS. Inspired by the success of \cite{exarchos2018stochastic}, we select the control computed from the last run as the importance sampling term. Then the IS  term in FSDE is defined as $m_s:=\Sigma K=GU_{*,0}$. New drift term is modified as $\tilde{\mu}_s=\mu_s+m_s$ with Lipschitz constant $\tilde{\mu}_x$ and $\tilde{\mu}_u$. 
In the BSDE, Then IS term is written as $k_s=Z_sK_s=Z_s\Gamma U_{*,0}=V_x G U_{*,0}$, and the modified $\tilde{H}_s=H_s+k_s$. Here we formally write the Lipschitz constant for IS terms in FSDE and BSDE are:
	\begin{equation}\label{fnc:lipz-impt-samp}
		\begin{aligned}
            |m_s(t,\vx_1,\vu_1)-m_s(t,\vx_2,\vu_2)|^2&\leq m_x|\vx_1-\vx_2|^2+m_{u}|\vu_1-\vu_2|^2,\\
            |k_s(t,\vx_1,\vz_1)-k_s(t,\vx_2,\vz_2)|^2&\leq k_x|\vx_1-\vx_2|^2+k_z|\vz_1-\vz_2|^2,\\
		\end{aligned}
	\end{equation}
Similar to the proof (\ref{lem:lemma1-proof}), we can have,
\begin{equation}
	\begin{aligned}
		|\delta \mX_t|^2&\leq e^{(I+\tilde{\mu}_x+\tilde{\mu}_u u_x+\Sigma_x)(t-t_0)}|\vx_1-\vx_2|^2\\
		&= e^{\tilde{\xi}(t-t_0)}|\vx_1-\vx_2|^2
	\end{aligned}
\end{equation}
Where $\tilde{\xi}=I+\tilde{\mu}_x+\tilde{\mu}_u u_x+\Sigma_x$.

And for Y term we will have,
\begin{equation}
	\begin{aligned}
		\mathbb{E}[|\delta \mY_t|^2]&=\mathbb{E}|\delta \mY_T|^2+\int_{t}^{T}\mathbb{E}\left[2(\delta H_s\delta Y_s+\delta k_s\delta Y_s)-\delta ||Z_s\T||_F^2\right]\rd s\\
		&\leq g_x\mathbb{E}|\delta \mX_T|^2+\int_{t}^{T} 2H_z \mathbb{E}|\delta Y_s|^2+ (2H_z)^{-1} \mathbb{E}|\delta H_s|^2+2k_z \mathbb{E}|\delta Y_s|^2+ (2k_z)^{-1} \mathbb{E}|\delta k_s|^2-\mathbb{E}||Z_s||_F^2\rd s\\
        &\leq g_x\mathbb{E}|\delta \mX_T|^2+\int_{t}^{T} (2H_z+2k_z) \mathbb{E}|\delta Y_s|^2+ (2H_z)^{-1} \mathbb{E}\left[H_x |\delta \mX_s|^2+H_z||\delta Z_s||_F^2\right]\\
        &+ (2k_z)^{-1} \mathbb{E}\left[k_x |\delta \mX_s|^2+k_z ||\delta Z_s||_F^2\right]-\mathbb{E}||Z_s||_F^2\rd s\\
        &\text{Noticing that the drift term in BSDE w/o IS is $H_s=C^{i *}+V_xGU_{*,0}$ which is a lipschitz continous function,}\\
        &\text{while IS term is $k_s=V_xGU_{*,0}$. then we can have $k_x\leq H_x$. By replacing $k_x$ by $H_x$, it yields:}\\
        &\leq g_x\mathbb{E}|\delta \mX_T|^2+\int_{t}^{T} (2H_z+2k_z) \mathbb{E}|\delta Y_s|^2+ (2H_z)^{-1} \mathbb{E}\left[H_x |\delta \mX_s|^2+H_z||\delta Z_s||_F^2\right]\\
        &+ (2k_z)^{-1} \mathbb{E}\left[H_x |\delta \mX_s|^2+k_z ||\delta Z_s||_F^2\right]-\mathbb{E}||Z_s||_F^2\rd s\\
        &\leq g_x\mathbb{E}|\delta \mX_T|^2+\int_{t}^{T} (2H_z+2k_z) \mathbb{E}|\delta Y_s|^2+ \frac{H_x}{2}(H_z^{-1}+k_z^{-1}) \mathbb{E} |\delta \mX_s|^2+\mathbb{E}||\delta Z_s||_F^2-\mathbb{E}||Z_s||_F^2\rd s\\
        &= g_x\mathbb{E}|\delta \mX_T|^2+\int_{t}^{T} (2H_z+2k_z) \mathbb{E}|\delta Y_s|^2+\frac{H_x}{2}(H_z^{-1}+k_z^{-1}) \mathbb{E} |\delta \mX_s|^2\rd s\\
		&\text{by Gronwall's inequality}\\
		|\delta \mY_t|^2&\leq e^{2(H_z+k_z)(T-t)}\left[g_x e^{\tilde{\xi}(T-t_0)}+H_x(H_z^{-1}+k_z^{-1})\frac{e^{\tilde{\xi}(T-t_0)}-e^{\tilde{\xi}(t-t_0)}}{2 \Sigma_x}\right]|\vx_1-\vx_2|^2\\
	\end{aligned}
\end{equation}
When $t_0=0$, we have,
\begin{equation}
	\begin{aligned}
    |\delta Y_T|^2&\leq g_x e^{\tilde{\xi} T}|\vx_1-\vx_2|^2\\
    &=\tilde{L}_1 |\vx_1-\vx_2|^2\\
    |\delta Y_0|^2&\leq e^{2(H_z+k_z)T}\left[ g_x e^{\tilde{\xi} T}+H_x(H_z^{-1}+k_z^{-1})\frac{e^{\tilde{\xi} T}-1}{2\Sigma_x}\right]|\vx_1-\vx_2|^2\\
    &=\tilde{L}_2 |\vx_1-\vx_2|^2\\
	\end{aligned}
\end{equation}
Where $\tilde{\xi}=I+\tilde{\mu}_x+\tilde{\mu}_u u_x+\Sigma_x$. Following arguments in \cite{ma2002representation}, one further has,
	\begin{equation}
		\begin{aligned}
            ||Z_t||_S^2\leq ||\Sigma||_S^2 ||\nabla_x Y_t||_S^2\leq M_\Sigma \tilde{L}_2
		\end{aligned}
	\end{equation}
\end{proof}

\subsection{Proof of Theorem \ref{prop: diff Y bound}}\label{appendix:prop-diff-Y proof}
According to the result in Lemma.\ref{Lem: Y Bound} with the assumption that the initial state dataset $\mathcal{D}$ are identical for FBSDE w/ and w/o importance sampling.
\begin{equation}
	\begin{aligned}
    |\delta Y_T|^2&\leq g_x e^{\xi T}|\vx_1-\vx_2|^2\\
    &=L_1 |\vx_1-\vx_2|^2\\
    |\delta Y_0|^2&\leq e^{H_zT}\left[ g_x e^{\xi T}+H_x\frac{e^{\xi T}-1}{H_z \Sigma_x}\right]|\vx_1-\vx_2|^2\\
    &=L_2 |\vx_1-\vx_2|^2\\
    \xi&=I+\mu_x+\mu_u u_x+\Sigma_x\\
	\end{aligned}
\end{equation}
Similarly, According to Lemma.\ref{Lem:lemma1-impt-samp}, one have,
\begin{equation}
	\begin{aligned}
        |\delta Y_T|^2&\leq g_x e^{\tilde{\xi} T}|\vx_1-\vx_2|^2\\
        &=\tilde{L}_1 |\vx_1-\vx_2|^2\\
        |\delta Y_0|^2&\leq e^{2(H_z+k_z)T}\left[ g_x e^{\tilde{\xi} T}+H_x(H_z^{-1}+k_z^{-1})\frac{e^{\tilde{\xi} T}-1}{2\Sigma_x}\right]|\vx_1-\vx_2|^2\\
        &=\tilde{L}_2 |\vx_1-\vx_2|^2\\
        \tilde{\xi}&=I+\tilde{\mu}_x+\tilde{\mu}_u u_x+\Sigma_x\\
	\end{aligned}
\end{equation}

We have $\tilde{\mu}_x =\mu_x+m_x \geq \mu_x$ and $\tilde{\mu}_u =\mu_u+m_u \geq \mu_u$, where $m_x$ and $m_u$ are the Lipschitz constants for $m_s$ w.r.t. $\vx$ and $\vu$ defined in equation.\ref{fnc:lipz-impt-samp} . Then we have $\tilde{\xi}\geq \xi$ which leads to $\tilde{L}_1\geq L_1$. 
Noticing that the drift term in BSDE w/o IS is $H_s=C^{i *}+V_xGU_{*,0}$, while IS term is $k_s=V_xGU_{*,0}$. then we can have $k_x\leq H_x$, and $k_z\leq H_z$ which leads to
\begin{equation}
	\begin{aligned}
    \frac{1}{2}(\frac{1}{H_z}+\frac{1}{k_z})>\frac{1}{H_z}
	\end{aligned}
\end{equation}
We have known that $\tilde{\mu}_x\geq \mu_x$ and $\frac{1}{2}(\frac{1}{H_z}+\frac{1}{k_z})>\frac{1}{H_z}$. Then we can have $\tilde{L}_1\geq L_1$ and $\tilde{L}_2 \geq L_2$ strictly.
\section{Invariant Layer Introductions and Implementation Techniques}\label{Appendix:invar-layer}
\subsection{Invariant Mapping}
A function $f$ maps its domain from $\mathcal{X}$ to $\mathcal{Y}$. Domain $\mathcal{X}$ is a vector space $\sR^{d}$ and $\mathcal{Y}$ is a continuous space $\sR$. Assume the function takes a set as input: $\sX=\left\{ x_1...x_N \right\}$, then the function $f$ is invariant if it satisfies property \ref{appendix:invar-property}.
\newlist{Properties}{enumerate}{2}
\setlist[Properties]{label=Property \arabic*.,itemindent=*,font=\textbf}
\begin{Properties}\label{appendix:invar-property}
  \item A function $f:\mathcal{X}\rightarrow\mathcal{Y}$ defined on sets is permutation invariant to the order of objects in the set. i.e. For any permutation function $\pi$: $f(\left\{ x_1...x_N \right\})=f(\left\{ x_{\pi(1)}...x_{\pi(N)} \right\})$
\end{Properties}
In this paper, we discuss the case when $f$ is a neural network only.
\begin{theorem}\citep{zaheer2017deep}
$\mX$ has elements from countable universe. A function $f(\mX)$ is a valid permutation invariant function, i.e invariant to the permutation of $\mX$, iff it can be decomposed in the form $\rho(\sum_{x\in\mX}\phi(x))$, for appropriate functions $\rho$ and $\phi$.
\end{theorem}
In the symmetric multi-agent system, each agent is not distinguishable. This property gives some hints about how to extract the features of the $-i$th agents using a neural network. The states of the $-i$th agents can be represented as a set: $\mX=\left\{X_1,X_2,...,X_{i-1},X_{i+1},...,X_{N}\right\}$. We want to design a neural network $f$ which has the property of permutation invariance. Specifically, $\phi$ is represented as a one layer neural network and $\rho$ is a common nonlinear activation function, and the invariant layer module is shown in Fig.\ref{fig:invar_layer}.
\subsection{Feature Extractor with Invariant Layer Architecture}\label{Appendix:feature Extractor arch}
The architecture of feature extractor with invariant layer is described  in Fig.~\ref{fig:invar_layer}.
\subsection{Invariant Layer Techniques}
Noticing that all the agents has the access to the global states, we define the state input features of invariant layer for the $i$th agent as:
	\begin{equation}\label{eq:Partial_Observable}
		\mX_{t}=\left\{X_i,X_1,X_2...,X_{i-1},X_{i+1},...X_N\right\},
	\end{equation}
	with shape of $[B,N,N_x]$, where $B$ is the batch size, $N_x$ is the dimension of the observed states. In the other word, \textbf{we always put own feature in the first position}. For each agent $i$, there will exist such a feature tensor, then for the feature extractor, the shape of input is $[BS,N,N,N_x]$. Therefore, the shape of input tensor will become $[BS,N,N-1,N_x]$ for invariant layer. where $N$ is the number of agents. First,we could use neural network to map the observed states to the feature space with dimension $N_f$. Then the shape of the tensor will become $[BS,N,N-1, N_f]$. After summing up the features of all the element in the set, the dimension of the tensor would reduce to $[BS,N,1, N_f]$, and we denote this feature tensor as $F$. However, the memory complexity is $\mathcal{O}(N^2 \times N_f)$ which is not tolerable when the number of agent $N$ increases. Alternatively, we can simply map the tensor $\mX_t$ whose dimension is $[BS,N,N_f]$ into the desired feature dimension $N_f$, then the shape of the tensor would become $[BS,N,N_f]$, and we denote this tensor as $F^{\prime}$. Now we create another tensor which is the average of features of element in set with size $[BS,1,N_f]$ and we denote it to be $\bar{F}^{\prime}$. Then we compute $\hat{F}=(\bar{F}^{\prime}\times N-F^{\prime})/(N-1)$ which has size of $[BS,N,N_f]$. We can find that $\hat{F}=F$, and the memory complexity of computing $\hat{F}$ is just $\mathcal{O}(N\times N_f)$. The derivation only holds when the system is symmetric and the agents are not distinguishable. The technique can be extended to higher state dimension for individual agent.
\begin{figure*}[tbp]
\centering
\includegraphics[width=0.7\linewidth]{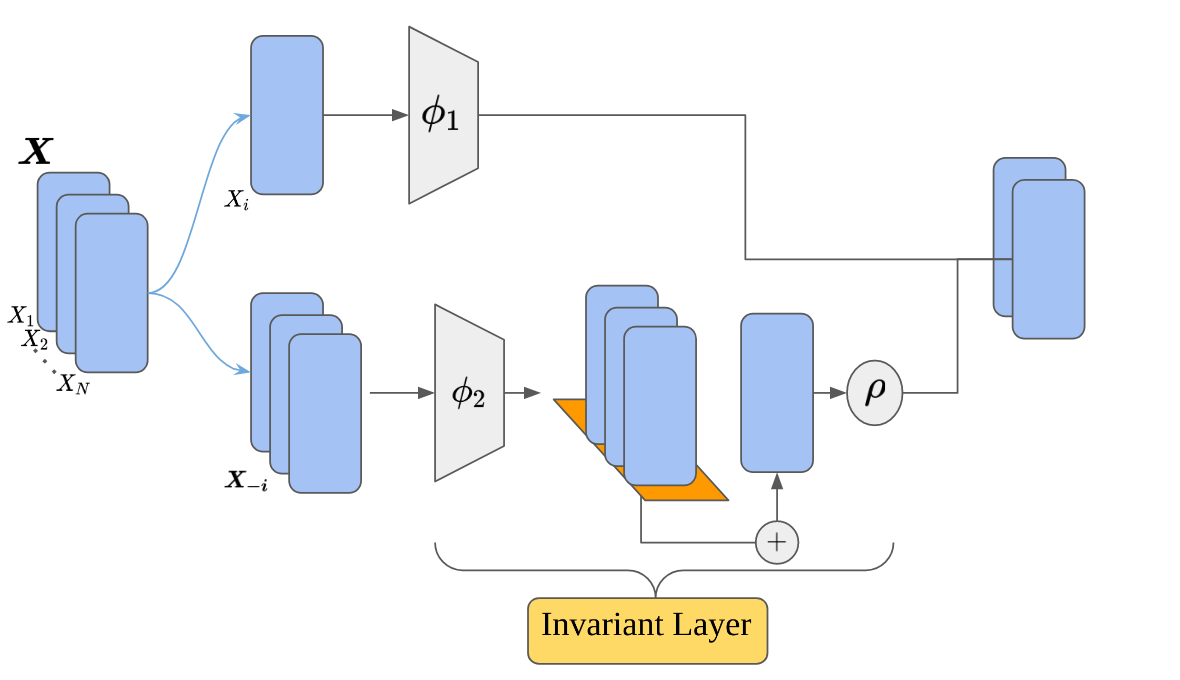}
\caption{Feature extractor architecture. $\phi$ represents fully connected neural network. $\rho$ is the ReLU activation function.}
\label{fig:invar_layer}
\end{figure*}
\section{Experiment Configurations}
\label{appendix:config}
This section elaborates the experiment configurations for $\S$\ref{sec:simulation}. For all the simulation, the number of SGD iteration is fixed as $N_{SGD}=100$. We are using Adam as optimizer with 1E-3 learning rate for all simulations. \\
\subsection{Inter-bank Experiments}
In section $\S$\ref{subsec:interbank}, For the prediction of initial value function $V_0^i$, all frameworks are using 2 layers feed forward network with 128 hidden dimension. For the baseline framework, we followed the suggested configuration motioned in \cite{han2018solving}. At each time steps, $V_{x}^i$ is approximated by three layers of feed forward network with 64 hidden dimensions. We add batch norm \cite{ioffe2015batch} after each affine transformation and before each nonlinear activation function. For Deep FBSDE with LSTM backbone, we are using two layer LSTM parametrized by 128 hidden state. If the framework includes the invariant layer, the number of mapping features is chosen to be 256. The hyperparameters of the dynamics is listed as following:
\begin{equation}
\begin{aligned}
 a=0.1, \ q=0.1,\  c=0.5,\  \epsilon=0.5,\  \rho=0.2,\  \sigma=1, T=1.
\end{aligned}    
\end{equation}
In the simulation, the time horizon is separated into 40 time-steps over 1 second by Euler method. Learning rate is chosen to be 1E-3 which is the default learning rate for Adam optimizer. The initial state for each agents are sampled from the uniform distribution $[-\delta_0,\delta_0]$. Where $\delta_0$ is the constant standard deviation of state $\mX(t)$ during the process as described in \cite{han2019deep}. In the evaluation, we are using 256 new sampled trajectory which are different from training trajectory to evaluate the performance. The number of stage is set to be 100 which is enough for all framework to converge.\\
\subsection{Belief Space Car Racing}
In $\S$\ref{sec:car racing result}, the hyperparameter is listed as following:
\begin{equation}
\begin{aligned}
 c_{drag}=0.01, \ L=0.1,\  c=0.5,\ T=10.0
\end{aligned}    
\end{equation}
The observation noise is sampled from Gaussian noise $m \sim \mathcal{N}(0,0.1\mI)$. The time horizon is enrolled into 100 time-steps by Euler method. In this experiments, the initial value $V_i$ is approximated a single trainable scale and $V_{x,i}(t)$ is approximated by two layers of LSTM parametrized with 32 hidden dimensions. The number of stage is set to be 10.
\section{FBSDEs and Analytical Solution for Inter-Bank Borrowing/Lending Problem}
\subsection{FBSDEs for Inter-Bank Borrowing/Lending Problem}\label{Appendix:FBSDEs for Inter-bank}
By plugging the running cost (\ref{fnc:inter-bank running cost}) to the HJB (\ref{Final_form_HJB}) function, one can have,
      \begin{equation}
      \begin{aligned}
         V_{t}^i+\inf \limits_{U_i \in \mathcal{U}_i}\left[\sum_{j=1}^{N}[a(\bar{X}-X_j)+U_j^2]V_{x_j}+\frac{1}{2}U_i^{2}-qU_{i}(\bar{X}-X_i)+\frac{\epsilon}{2}(\bar{X}-X_i)^2\right]+\\
         \frac{1}{2}\tr(V_{xx}^i\Sigma\Sigma\T)=0.\\
      \end{aligned}    
      \end{equation}
   By computing the infimum explicitly, the optimal control of player $i$ is:$U_i^*(\mX,t)=q(\bar{X}-X_i)-V_{x}^{i}(\mX,t)$. The final form of HJB can be obtained as
      \begin{equation}
      \begin{aligned}
         V_{t}^i+\frac{1}{2}\tr(V_{xx}^i\Sigma\Sigma\T)+a(\bar{X}-X_i)V_{x}^i+\sum_{j\neq i}[a(\bar{X}-X_j)+U_j]V_{x}^j\\
         +\frac{\epsilon}{2}(\bar{X}-X_i)^2-\frac{1}{2}(q(\bar{X}-X_i)-V_{x}^i)^2=0
      \end{aligned}\label{eq:HJB_interbank} 
      \end{equation}
   Applying Feynman-Kac lemma to \eqref{eq:HJB_interbank}, the corresponding FBSDE system is
      \begin{equation}
      \begin{aligned}
         \rd  \mX(t)&=(f(\mX(t),t) +G(\mX(t),t)\vu(t))\rd t+\Sigma(t,
         \mX(t))\rd\boldsymbol{W_t},\quad \mX(0)=\vx_0\\
           \rd V^i&=-[\frac{\epsilon}{2}(\bar{X}-X_i)^2-\frac{1}{2}(q(\bar{X}-X_i)-V_{x}^i)^2+U_i]\rd t+V_{x}^{i\T} \Sigma dW, \quad V(T)=g(\mX(T)).
      \end{aligned}    
      \end{equation}
\subsection{Analytical solutions for Inter-Bank Borrowing/Lending Problem}
\label{Appendix:Analytical Solution}
The analytical solution for linear inter-bank problem was derived in \cite{carmona2013mean}. We provide them here for completeness. Assume the ansatz for HJB function is described as:
	\begin{equation}
	\begin{aligned}
        V_i(t,\mX)=\frac{\eta(t)}{2}(\bar{X}-X_i)^2=\mu(t) \quad i\in \sI
	\end{aligned}
	\end{equation}
Where $\eta(t),\mu(t)$ are two scalar functions. The optimal control under this ansatz is:
	\begin{equation}
	\begin{aligned}
        U_i^{\star}(t,\mX)=\left[q+\eta(t)(1-\frac{1}{N})\right](\bar{X}-X_i)
	\end{aligned}
	\end{equation}
By pluginging the ansatz into HJB function derived in equation (\ref{eq:HJB_interbank}), one can have,
	\begin{equation}
	\begin{aligned}
        \dot{\eta}(t)&=2(a+q)\eta(t)+(1-\frac{1}{N^2})\eta^2(t)-(\epsilon-q^2), \quad \eta(T)=c,\\
        \dot{\mu}(t)&=-\frac{1}{2}\sigma^2(1-\rho^2)(1-\frac{1}{N})\eta(t), \quad \mu(T)=0.
	\end{aligned}
	\end{equation}
There exists the analytical solution for the Riccati equation described above as,
	\begin{equation}
	\begin{aligned}
        \eta(t)=\frac{-(\epsilon-q^2)(e^{(\delta^{+}-\delta^{-})(T-t)}-1)-c(\delta^{+}e^{(\delta^{+}-\delta^{-})(T-t)}-\delta^{-})}{(\delta^{-}e^{(\delta^{+}-\delta^{-})(T-t)}-\delta^{+})-c(1-1/N^2)(e^{(\delta^{+}-\delta^{-})(T-t)})-1}.
	\end{aligned}
	\end{equation}
Where $\delta^{\pm}=-(a+q)\pm\sqrt{R}$ and $R=(a+q)^2+(1-1/N^2)(\epsilon-q^2)$

\section{Additional Tables and Figures}

\subsection{Evaluation Loss with different number of agents}\label{appendix:eval-loss-inter-bank}
Fig.\ref{fig:eval-comparsion-inter-bank} shows the comparison of SDFP-FBSDE and baseline by evaluation loss.
\begin{figure*}[h]
\centering
\includegraphics[width=0.4\linewidth]{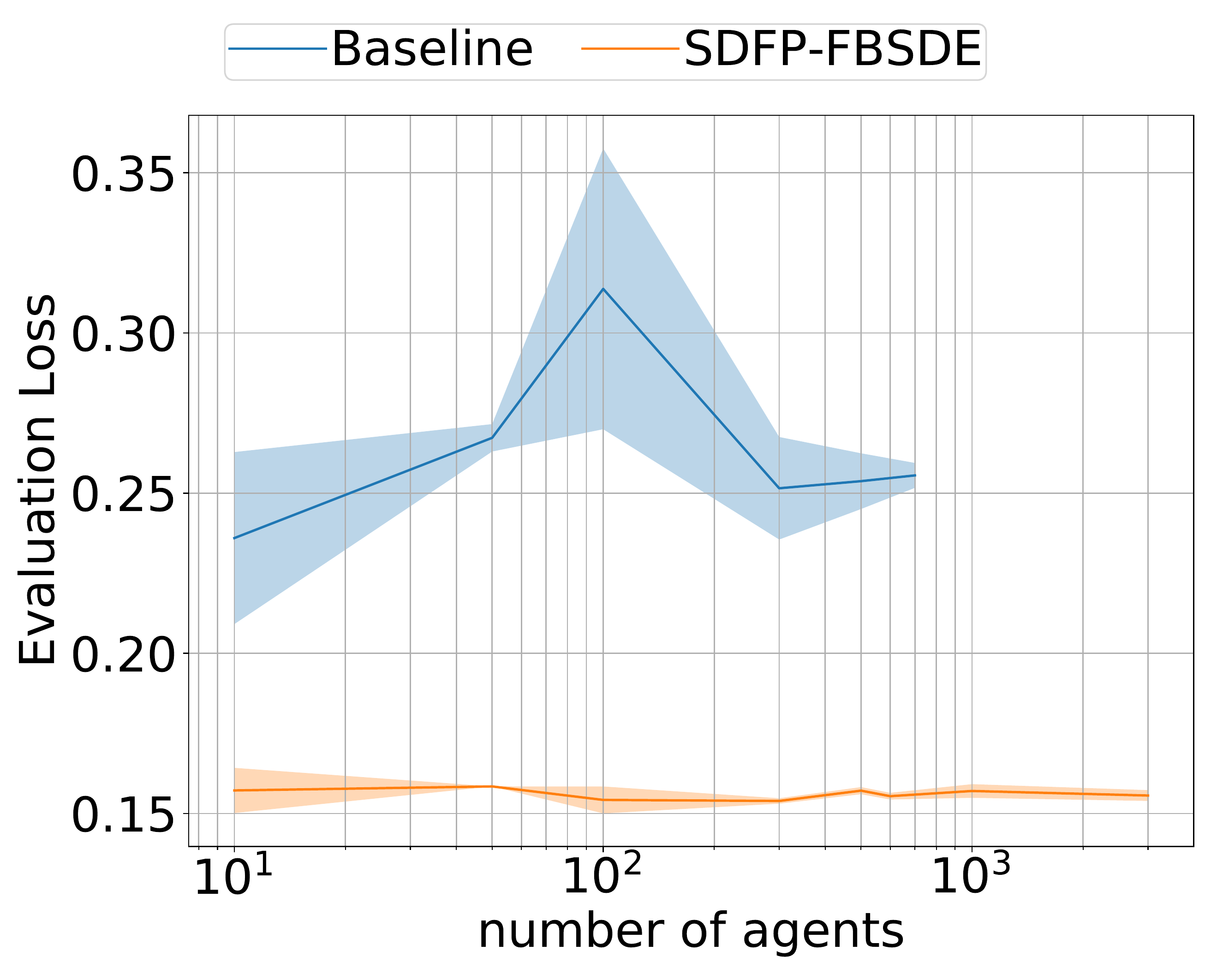}
\vspace{-13 pt}
\caption{Comparison of SDFP-FBSDE and Baseline for inter-bank problem with different number of agents evaluated on evaluation loss(\ref{fnc:eval-loss}).}
\vspace{-7.5 pt}
\label{fig:eval-comparsion-inter-bank}
\end{figure*}
\subsection{Superlinear Inter-Bank Plots}\label{Appendix:superlinear}
Fig.\ref{fig:superlinear} demonstrates the performance difference between Baseline and our algorithm. One can find that our algorithm convergence faster and better than baseline. Since in the superlinear case, the influence of control term in the forward dynamics is mitigated, then the final performances are similar.
\begin{figure*}[h]
\centering
\includegraphics[width=0.6\linewidth]{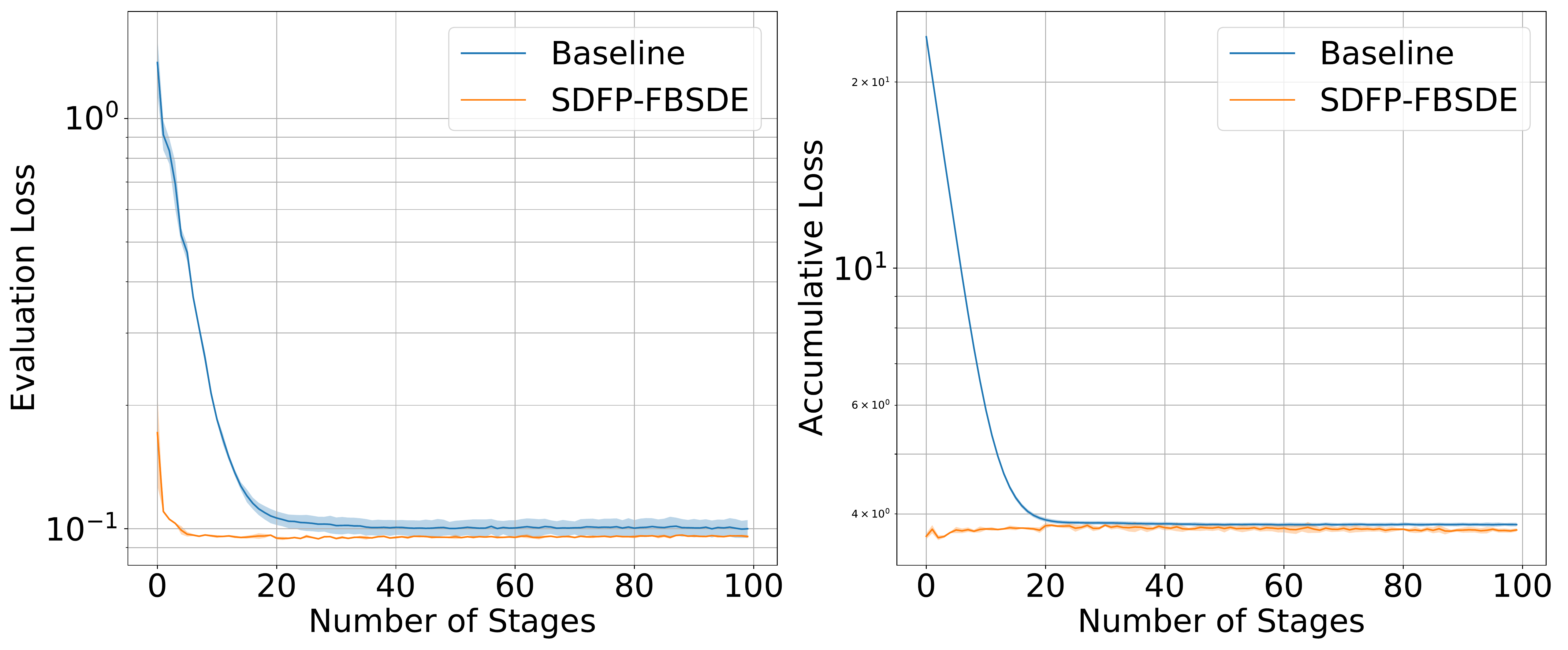}
\vspace{-13 pt}
\caption{Comparison of SDFP-FBSDE and Baseline for inter-bank problem in superlinear case with evaluation loss (\ref{fnc:eval-loss}) and cumulative loss(\ref{fnc:acc-loss})}
\vspace{-7.5 pt}
\label{fig:superlinear}
\end{figure*}
\subsection{Posterior Plot of Car Racing}\label{Appendix:car-racing post}
Fig.\ref{fig:Car-racing post} illustrates the trajectory of single game with posterior estimated by each car. One can find that the variance does not blow up, and both of two cars are staying in the track.
\begin{figure*}[h]
\centering
\includegraphics[width=0.5\linewidth]{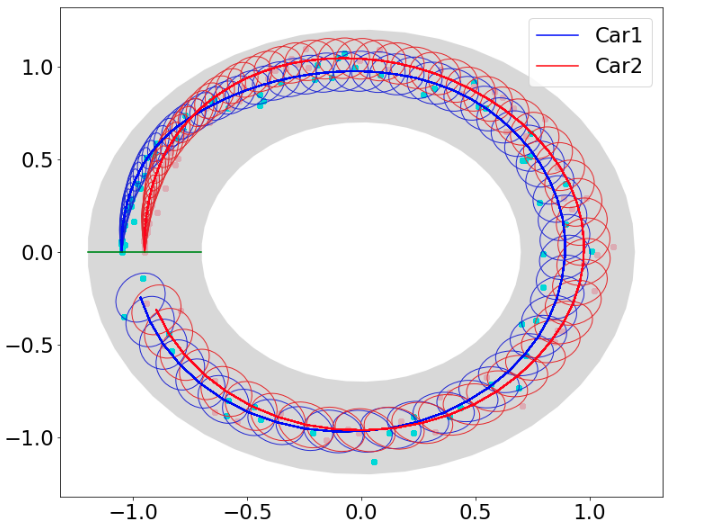}
\vspace{-13 pt}
\caption{Car racing plot with posterior trained with SDFP-FBSDE. The competition loss is turned off}
\vspace{-7.5 pt}
\label{fig:Car-racing post}
\end{figure*}

\subsection{DFP-FBSDE Framework}\label{appendix:NN_arch}
In this subsection, we demonstrate the framework of Deep Fictitious Play FBSDE (DFP-FBSDE) in Fig.\ref{fig:NN_arch}. Each $NN$ (blue box) represents for the FBSDE module shown in Fig.\ref{fig:architecture}.
\begin{figure*}[h]
\centering
\includegraphics[width=0.8\linewidth]{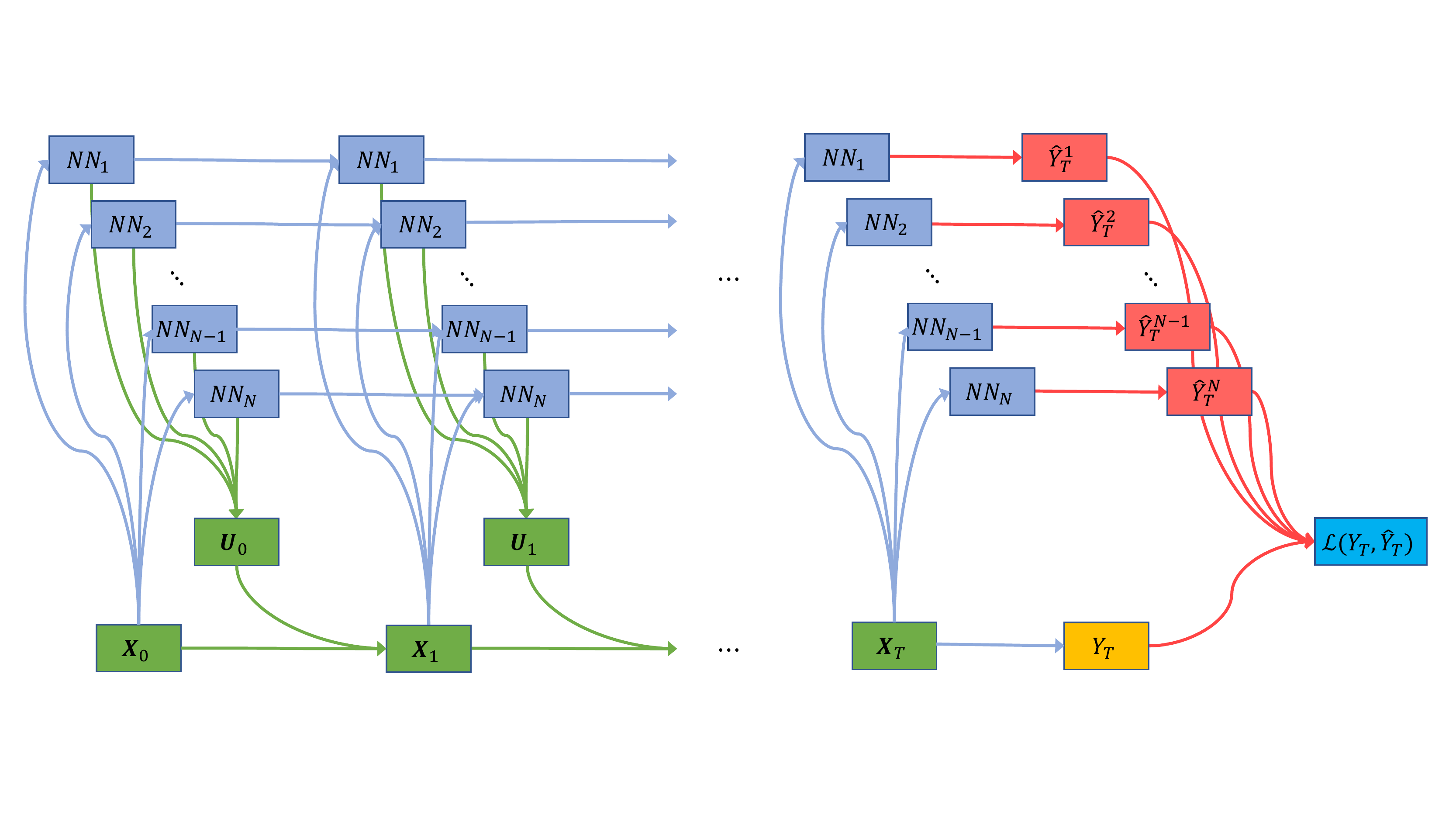}
\vspace{-13 pt}
\caption{SDFP framework for $N$ Players. Each NN block has the architecture in Fig.~\ref{fig:architecture}}
\vspace{-7.5 pt}
\label{fig:NN_arch}
\end{figure*}

\section{Belief Car Racing}
\label{appendix:racing}
The framework for the racing problem is trained with batch size of 64, and 100 time steps over a time horizon of 10 seconds.
\subsection{Continuous Time Extended Kalman Filter}\label{appendix:EKF}
The Partial Observable Markov Decision Process is generally difficult to solve within infinite dimensional space belief. Commonly, the Value function does not have explicit parameterized form. Kalman filter overcome this challenge by presuming the noise distribution is Gaussian distribution. In order to deploy proposed Forward Backward Stochastic Differential Equation (FBSDE) model in the Belief space, we need to utilize extended Kalman filter in continuous time \cite{jazwinski1970stochastic} correspondingly. Given the partial observable stochastic system:
	\begin{equation}\label{fnc:EKF-dyn}
		\frac{\rd x}{\rd t}=f(x,u,w,t), \quad \text{and} \quad	z=h(x,v,t)
	\end{equation}
	
	Where $f$ is the stochastic state process featured by a Gaussian noise $	w\sim \mathcal{N}(0,Q)$, $h$ is the observation function while $v\sim \mathcal{N}(0,R)$ is the observation noise. Next, we consider the linearization of the stochastic dynamics in eq.(\ref{eq:Partial_Observable}) represented as follows:
	\begin{equation}
		A=\frac{\partial f}{\partial x}\bigg|_{\hat{x}}, L=\frac{\partial f}{\partial w}\bigg|_{\hat{x}},	C =\frac{\partial h}{\partial x}\bigg|_{\hat{x}}, M=\frac{\partial h}{\partial v}\bigg|_{\hat{x}},
		\tilde{Q}=LQL\T, 
		\tilde{R}=MRM\T \label{Kalman_components}
	\end{equation}
	 one can write the posterior mean state $\hat{x}$ and prior covariance matrix $P^{-}$ estimation update rule by \cite{simon2006optimal}:
	\begin{equation}
	\begin{aligned}
		\hat{x}(0)&\mathbb{E}[x(0)], \quad 	P^{-}(0)=\mathbb{E}[(x(0)-\hat{x})(x(0)-\hat{x})\T]\\
		K&=PC\T\tilde{R}^{-1}\\
		\dot{\hat{x}}&=b(\hat{x},u,w,t)=f(\hat{x},u,w_0,t)+K[z-h(\hat{x},v_0,t)]\\
		\dot{P}^{-}&=AP^{-}+P^{-}A\T +\tilde{Q}-P^{-}C\T \tilde{R}^{-1}CP^{-}
	\end{aligned}
	\end{equation}
	We follow the notation in \citep{simon2006optimal}, where $x$ is the real state, $\hat{x}$ is the mean of state estimated by Kalman filter based on the noisy sensor observation, $P^{-}$ represents for the covariance matrix of the estimated state, nominal noise values are given as $w_0=0$ and $v_0=0$, where superscript + is the posterior estimation and $-$ is the prior estimation. Then we can define a Gaussian belief dynamics as $\vb(\hat{x}_k,P^{-}_{k})$ by the mean state $\hat{x}$ and variance $P^{-}$ of normal distribution $\mathcal{N}(\hat{x}_k,P^{-}_{k})$

	The belief dynamics results in a decoupled FBSDE system as follows:
	\begin{equation}
	\begin{aligned}
		\rd \vb_{k}&=g(\vb_{k},\mU_{k},0)\rd t+\Sigma(\vb_{k},\mU_{k},0)\rd W, \rd W \sim \mathcal{N}(0,I)\\
		\rd V&=-C^{i^{\star}}\rd t+V_x^{i\T}\Sigma \rd W
	\end{aligned}
	\end{equation}
	where:
	\begin{equation}\label{fnc:EKF-FBSDE}
	\begin{aligned}
		g(\mathbf{b}_k,\mU_k)&= 	\left[
										\begin{matrix}
											b(t,\mX(t),U_{i,m}(t);\mU_{-i,m}) \\
										vec(A_kP^{-}_k+P^{-}_kA\T_k +\tilde{Q}_k-P^{-}_kC\T_k
																			\tilde{R}^{-1}_kC_kP^{-}_k) \\
										\end{matrix}
										\right]\\
		\Sigma(\mathbf{b}_k,\mU_k)&= 	\left[
			\begin{matrix}
			\sqrt{K_kC_kP_k^{-}}\rd t \\
			\mathbf{0} \\
			\end{matrix}
			\right]\\
		V(T)&=g(\mX(T))\\
		\hat{\mX}(0)&=\mathbb{E}[\mX(0)]\\
		P^{-}(0)&=\mathbb{E}[(\mX(0)-\hat{\mX})(\mX(0)-\hat{\mX})\T]\\
   \end{aligned}
   \end{equation}
In the car racing case, the dynamic function $f(\cdot,\cdot)$ in eq.\ref{fnc:EKF-dyn} is described as, 
      \begin{equation}
      \begin{aligned}
           \rd \mX &= (f(\mX)+G(\mathbf{X})\mathbf{U})\rd t + \Sigma(\mathbf{X})\rd \mW, \quad \mathbf{z} = h(\mathbf{X}) + m\\
           f(\mathbf{X}) &= \begin{bmatrix}v\cos\theta\\ v\sin\theta\\-c_{\text{drag}}v\\0 \\ \end{bmatrix},~~ G(\mathbf{X})= \Sigma(\mathbf{X}) = \begin{bmatrix}0 & 0 \\ 0 & 0 \\ 1 & 0 \\ 0 & v/L \end{bmatrix}, ~~ h(\mX) = \vx \\
      \end{aligned}    
      \end{equation}
      
   Where $\rd \mW$ is standard Brownian motion.
   \subsection{Cost Functions}\label{appendix:car-racing cost}
   We consider the problem of two cars racing on a circular track. The cost function of each car is designed as
   \begin{equation*}
       J_t =  \underbrace{\exp\big(\big|\frac{x^2}{a^2}+\frac{y^2}{b^2}-1\big|\big)}_{\textit{track cost}} + \underbrace{\text{ReLU}\big(-v\big)}_{\textit{velocity cost}} + \underbrace{\exp\big(-d)}_{\textit{collision cost}}
   \end{equation*}
   Where $d$ is Euclidean distance between two cars. We use continuous time extended Kalman Filter to propagate belief space dynamics described in \eqref{fnc:EKF-FBSDE}.
   
   We introduce the concept competitive game by using an additional competition cost:
   \begin{equation*}
   \begin{aligned}
             J_{competition}&=\exp(-\left[
                                 \begin{matrix}
                                    cos(\theta)\\
                                     sin(\theta)\\
                                 \end{matrix}
                                 \right]\T
                                 \left[
                                 \begin{matrix}
                                    x_1-x_2\\
                                     y_1-y_2\\
                                 \end{matrix}
                                 \right])\
   \end{aligned}
   \end{equation*}
   Where $x_i,y_i$ is the $x,y$ position of the $i$th car. When the $i$th car is leading, the competition loss will be minor, and it will increase exponentially when the car is trailing.\\
   Thanks to the decoupled BSDE structure, each car can measure this competition loss separately and optimize the value function individually.
\section{Hardware}\label{Appendix:hardware}
All simulations are run on
\begin{enumerate}
    \item Nvidia RTX TITAN
    \item Nvidia GTX TITAN BLACK
\end{enumerate}
\section{Algorithm}\label{appendix:Algo}
\begin{algorithm}[h]
\caption{Scalable Deep Fictitious Play FBSDE}
\begin{algorithmic}[1]
\STATE \textbf{Hyper-parameters}:$N$: Number of players; $T$: Number of timesteps; $M$: Number of stages in fictitious play; $N_{gd}$: Number of gradient descent steps per stage; $\mU_{0}$: the initial strategies for players in set $\sI$; $B$: Batch size; $\Delta t$: time discretization (Total time/Number of timesteps); $\pi$: Permutation function (\ref{appendix:invar-property}).
\STATE \textbf{Parameters}:$\phi$: Network weights for Initial Value (IV) prediction $f_{IV}(\cdot)$; $\theta$: Weights and bias of Backbone and Feature extractor (BF) $f_{BF}(\cdot)$.
\STATE Initialize trainable papermeters:$\theta^{0}$, $\phi^{0}$
\FOR{$m \leftarrow 1$ to $M$}
\STATE Generate $B$ sample $\vx_0$ and $B\times T$ Noise $\Delta \vw \sim \mathcal{N}(0, \mI\Delta t)$.
\FOR{$l \leftarrow 0$ to $N_{gd}-1$}
\FOR{$t \leftarrow 0$ to $T-1$}
\IF{t==0}
\STATE Predict value function for $i$th player: $\hat{y}^i_0=f_{IV}(\vx_0;\phi^{m\times N_gd+l})$
\ELSE
\STATE Compute network prediction $\hat{z}_{i}$ for $i$th player: $\hat{z}_i=\Sigma_i\T f_{BF}(\vx_t;\theta^{m\times N_{gd}+l})$
\ENDIF
\STATE Compute $i$th optimal control :$u_{i}^*=-R^{-1}(\Gamma_i^\mathrm{T}z_i+Q_i^\mathrm{T}\vx_t)$
\STATE Infer $-i$th players' network prediction and stop the gradient for them: $\hat{\vz}_{-i}=\Sigma_{-i} \T f_{BF}(\pi(\vx_t);\theta^{m\times N_{gd}})$
\STATE Compute $-i$th optimal Control and stop the gradient for them: $\vu_{-i}^*=-R^{-1}_{-i}(\Gamma_{-i}^\mathrm{T}\hat{\vz}_{-i}+Q_{-i}^\mathrm{T}\vx_t)$
\STATE Propagate FSDE: $\vx_{t+1}=f_{FSDE}(\vx_t,u_{i}^*, \vu_{-i}^*,\Delta \vw_t,t)$\quad (\ref{fnc:impt-FSDE})
\STATE Propagate BSDE: $\hat{y}^i_{t+1}=f_{BSDE}(\hat{y}^i_t,\vx_t,u_{i}^*, \vu_{-i}^*,\hat{z}_i,\Delta \vw_t,t)$\quad(\ref{fnc:impt-BSDE})
\ENDFOR
\ENDFOR
\STATE Compute True terminal value $y_T^{i}=g^i(\vx_T)$
\STATE Compute loss: $\mathcal{L}(\hat{y}_T^i,y_T^i)=\frac{1}{B}||\hat{y}_T^i-y_T^i||_2^2$ \quad (\ref{fnc:eval-loss})
\STATE Gradient Update: $\theta^{l},\phi^{l}$
\ENDFOR
\label{code:recentEnd}
\end{algorithmic}
\end{algorithm}

\end{document}